\documentclass[12pt]{article}
\pdfoutput=1 

\usepackage{hyperref}       
\usepackage{url}            
\usepackage{booktabs}       
\usepackage{amsfonts}       
\usepackage{nicefrac}       
\usepackage{microtype}      
\usepackage[numbers]{natbib}
\usepackage{graphicx}
\usepackage{subfig}
\usepackage{amsmath,amssymb,amsthm}       
\usepackage{nicefrac,verbatim}       
\usepackage{microtype}      
\usepackage{color}
\usepackage{comment}
\usepackage{algorithm}
\usepackage[noend]{algpseudocode} 
\usepackage{wrapfig}
\usepackage{lipsum}

\usepackage[margin=1in]{geometry}
\usepackage{appendix}

\newcommand{\bone}{\mathbf{1}}

\newcommand{\rd}{\color{red}}
\newcommand{\bk}{\color{black}}

\def\E{\mathbb{E}}

\def\R{\mathbb{R}}

\def\sI{\mathcal{I}}

\def\sX{\mathcal{X}}

\def\1{\boldsymbol{1}}
\def\btheta{\boldsymbol{\theta}}
\def\bphi{\boldsymbol{\phi}}
\def\balpha{\boldsymbol{\alpha}}
\def\bx{\mathbf{x}}
\def\by{\mathbf{y}}

\def\bZ{\mathbf{Z}}

\newtheorem{thm}{Theorem}
\newtheorem{assum}{Assumption}
\newtheorem{lem}[thm]{Lemma}

\newtheorem{proposition}{Proposition}

\newtheorem{remark}{Remark}

\usepackage{authblk}

\newcommand\CoAuthorMark{\footnotemark[\arabic{footnote}]}

\title{\Large\textsc{Mini-batch Metropolis-Hastings MCMC with Reversible SGLD Proposal}}

\author[1]{Tung-Yu Wu\footnote{Equal contribution.}}
\author[2]{Y. X. Rachel Wang\protect\CoAuthorMark}
\author[1,3]{Wing H. Wong}

\affil[1]{Institute for Computational \& Mathematical Engineering, Stanford University}
\affil[2]{School of Mathematics and Statistics, University of Sydney}
\affil[3]{Department of Statistics, Stanford University}
\affil[ ]{\vskip1pt E-mail addresses: {\texttt  tungyuwu@stanford.edu, rachel.wang@sydney.edu.au, whwong@stanford.edu}}

\allowdisplaybreaks

\date{}
\begin{document}
	
	\maketitle
\begin{abstract}
	Traditional MCMC algorithms are computationally intensive and do not scale well to large data. In particular, the Metropolis-Hastings (MH) algorithm requires passing over the entire dataset to evaluate the likelihood ratio in each iteration. We propose a general framework for performing MH-MCMC using mini-batches of the whole dataset and show that this gives rise to approximately a tempered stationary distribution. We prove that the algorithm preserves the modes of the original target distribution and derive an error bound on the approximation with mild assumptions on the likelihood. To further extend the utility of the algorithm to high dimensional settings, we construct a proposal with forward and reverse moves using stochastic gradient and show that the construction leads to reasonable acceptance probabilities. We demonstrate the performance of our algorithm in both low dimensional models and high dimensional neural network applications. Particularly in the latter case, compared to popular optimization methods, our method is more robust to the choice of learning rate and improves testing accuracy. 
\end{abstract}

\section{Introduction}
\label{introduction}

Since its inception, Markov chain Monte Carlo (MCMC) sampling has been an indispensable tool in Bayesian modeling for obtaining parameter estimates and their uncertainty. However, traditional MCMC algorithms do not scale well to large data as they typically involve expensive computation using the full dataset. Additionally, scaling classical MCMCs toward modern high-dimensional applications can be problematic. The computational bottleneck led researchers to pursue lower accuracy, higher efficiency trade-offs such as variational inference. Despite its computational efficiency, theoretical guarantees for asymptotic convergence of variational approximations are given typically for specific models, and the objective function can contain multiple local optima trapping commonly used optimization algorithms \citep{blei2017, ghorbani2018, mukherjee2018}. In comparison, MCMC techniques have the potential to navigate non-convex surfaces and find better local optima in the process. As the amount of data continues to grow rapidly, the need for scalable MCMC methods for large-scale learning tasks remains critical. In this paper, we propose an MCMC algorithm that is scalable in both the size of the dataset and the dimension of the parameter space. Our algorithm leverages the traveling property of an MCMC sampler to find better solutions to optimization problems in machine learning. 


The search for scalable MCMC methods has largely proceeded in two directions. The first approach divides the data into manageable batches and performs MCMC on each batch in parallel. To collectively process the results, most methods either require different machines to communicate with each other in different rounds of MCMC iteration \citep{agarwal2011}, or combine the posterior distribution from each batch to approximate the target posterior \citep{neiswanger2013, wang2013, scott2016}. Our work follows the second line of approach, which uses subsamples, or mini-batches, of the full data in each iteration of the MCMC algorithm. The key in analyzing such an algorithm is to understand the noise and bias introduced by the mini-batches. 

The broad class of pseudo-marginal algorithms \citep{andrieu2009} use mini-batches of data to accelerate computation in the Metropolis-Hastings (MH) algorithm \citep{jacob2015, bardenet2015, maclaurin2014, quiroz2018}. The exact posterior (or some close approximation) is maintained by constructing an unbiased and nonnegative estimator, which can have a nontrivial form or require carefully chosen lower bound on the likelihood. 
Another class of methods performs approximate tests in the MH acceptance step using mini-batches. To control the approximation error, an adaptive approach is usually adopted to sequentially increase the size of a batch until an error bound is met \citep{bardenet2014, korattikara2014, chen2016}. 
Approaches based on non-reversible MCMC have also been proposed \citep{bierkens2019}. In practice, some of these methods were tested on large datasets with hundreds of parameters, but further scaling up in parameter dimension toward deep machine learning models would be challenging. 


In another direction, past few years have witnessed the rise of stochastic gradient based MCMC algorithms which have shown strong potential in large-scale machine learning applications. These algorithms are developed from diffusion-based MCMC and approximate the gradient with noisy estimates based on mini-batches of data (\cite{robbins1985}), a notable example being the Stochastic Gradient Langevin Dynamics (SGLD) and other variants \citep{welling2011, ahn2012, chen2014, li2016}. Many studies have since analyzed the convergence of SGLD by viewing the algorithm as a discrete-time simulation of a continuous stochastic differential equation (SDE) \citep{teh2016, raginsky2017}. Unlike algorithms such as MALA which uses the MH acceptance test to correct the errors in discretizing a continuous system (e.g. \cite{roberts1996}), SGLD completely avoids the costly computation of the MH ratio by using a shrinking step size. In practice, this implies the algorithm eventually converges to a local optimum.


We propose a general mini-batch MH algorithm whose invariant distribution approximates a tempered version of the target posterior. By augmenting the system with a variable related to the subsampling procedure, we show our algorithm is a reversible Markov chain thus has an invariant distribution. The idea of augmenting the system to sample a tempered posterior was also explored by \cite{li2017} to heuristically design a mini-batch Metropolis sampler, but their algorithm differs in the use of mini-batches and they did not offer theoretical support for the method. \cite{de2018} introduced a mini-batch Gibbs sampler capable of exact sampling from certain graphical models. Finally, a connection between tempering and subsample variance was also mentioned in \cite{bardenet2015}. Here, we provide a rigorous theoretical foundation for mini-batching in MH. We emphasize that our aim here is not Bayesian inference from the exact posterior. Rather, we exploit the tempered posterior with an efficient MCMC sampler to obtain better solutions from a global optimization.  



With mild assumptions on the likelihood and allowing the parameter dimension to grow at a suitable rate, we provide full theoretical analysis to i) show the invariant distribution of our algorithm approximately preserves the modes of the true posterior, which is an important property for optimization tasks, and ii) bound the distance between the invariant distribution and the tempered posterior. To further enhance the utility of our algorithm in high dimensional applications, we design a proposal function based on Reversible Stochastic Gradient Langevin Dynamic (RSGLD) to make the calculation of MH ratio computationally efficient while ensuring reasonable acceptance probability. We show that the proposal significantly enhances acceptance probability in regions with strong gradient information and explores flat regions in a way similar to random walk. Empirically, we demonstrate the tempering effect inherent to our algorithm helps the Markov chain jump out of local optima and travel between differently modes more easily. Most importantly, we show our mini-batch MH algorithm combined with the RSGLD proposal can be applied to efficiently train neural networks.

The rest of the paper is organized as follows. In Section~\ref{methods}, we introduce our algorithm and provide theoretical analysis of its stationary distribution. In the high dimensional setting, we also design a proposal function called RSGLD and show that adding the reverse move significantly increases the acceptance probability when the gradient is strong. In Section~\ref{sec:exp},  we demonstrate with an array of examples from simple Gaussian models to neural networks with $>10^5$ parameters that our algorithm combines the traveling property of an MCMC sampler and the computational efficiency of stochastic optimization methods, thus showing good promise for optimization tasks in deep machine learning applications. In the neural network examples, our algorithm shows higher accuracy overall and better stability for larger learning rates compared to other popular optimization methods.

\section{Methods}
\label{methods}
We first introduce our algorithm and outline its connection to tempering using an augmented variable. We then show under appropriate assumptions, the stationary distribution of the mini-batch MH approximately preserves the modes of the target posterior and is close to a tempered posterior in distribution. In the high dimensional setting, we design a proposal function that can navigate a complex surface guided by gradient information and ensure the acceptance probability does not diminish too quickly as the dimension grows. 

\subsection{MH MCMC with batch tempering (MHBT)}
\label{subsec_minibatch_mh}
Under the usual Bayesian setting, let $\bx=(x_1, \dots, x_n)\in \mathcal{X}^n$, $\mathcal{X}\subset\R^p$, be iid samples drawn from distribution $p(\cdot | \btheta^*)$, where $\btheta^*\in\Theta\subset\R^d$ denotes the parameters. Let $\pi_0(\btheta)$ be the prior on $\btheta$. We are interested in sampling from the target posterior $\pi(\btheta)\propto \pi_0(\btheta)\prod_{i=1}^{n} p(x_i | \btheta)$ using the MH algorithm. In each iteration of classical MH, given some proposal function $q(\cdot)$, a move from $\btheta$ to $\btheta'$ is accepted with probability given by the MH ratio,
$$
r(\btheta \to \btheta') = \min \left\{ 1, \frac{\pi(\btheta')q(\btheta' \to \btheta)}{\pi(\btheta)q(\btheta \to \btheta')} \right\}.
$$
For large $n$, the evaluation of $\pi(\cdot)$ is costly. Now denote  $\ell_i(\btheta) = \log p(x_i | \btheta)$, $\mu(\btheta) = \frac{1}{n}\sum_{i=1}^{n} \ell_i(\btheta)$, $\hat{\mu}_I(\btheta) = \frac{1}{|I|}\sum_{j\in I} \ell_{j}(\btheta)$ with $I\subset\{1, \dots, n\}=[n]$ being an index subset. Let $\sI_{m}$ be the collection of $I$ such that $|I|=m$. We will use $\hat{\mu}_I(\btheta)$ to approximate $\mu(\btheta)$.

We next derive our algorithm, \textit{MH MCMC with batch tempering (MHBT)}, using an augmented system \footnote{For simplicity of description, we assume the prior $\pi_0(\btheta)\propto 1$; the algorithm and theoretical results generalize with minor modifications to other priors for large $n$.}. 
Consider an auxiliary variable $\tau\in\{0,1\}^n$ with $I(\tau)=\{i: \tau_i=1\}$ and $|I(\tau)|=m$, then we can write $\hat{\mu}_{I(\tau)} = \frac{1}{m}\sum_{i=1}^n \ell_i(\btheta)\tau_i$. Jointly for $(\btheta, \tau)$, consider the proposal $q((\btheta, \tau) \to (\btheta', \tau')) = q(\btheta \to \btheta') \nu_{m,n}(\tau')$ and the target distribution  
\begin{align}
	\tilde{\pi}(\btheta, \tau) \propto e^{c_n \hat{\mu}_{I(\tau)}(\btheta)}\nu_{m,n}(\tau)
	\label{eq_aug_mh_ratio}
\end{align}
where $\nu_{m,n}$ is the uniform distribution over $\sI_{m}$ and $c_n$ is a scaling constant that will be explained soon. Performing the classical MH algorithm on the augmented pair $(\btheta, \tau)$ with the above proposal and $\tilde{\pi}$, simple algebra shows the acceptance probability is given by 
\begin{align}
	r((\btheta, \tau) \to (\btheta', \tau')) 	& =  \min \left\{ 1, \frac{\tilde{\pi}(\btheta', \tau') q((\btheta',  \tau') \to (\btheta,  \tau))}{\tilde{\pi}(\btheta,  \tau) q((\btheta,  \tau) \to (\btheta',  \tau'))} \right\}
	= \min \left\{ 1, \frac{q(\btheta' \to \btheta) e^{c_n \hat{\mu}_{I(\tau')}(\btheta')} }{q(\btheta \to \btheta') e^{c_n \hat{\mu}_{I(\tau)} (\btheta)}} \right\},
\end{align}
which can be calculated efficiently using a new mini-batch $I(\tau')$ of the data. Since the stationary distribution of this Markov chain is $\tilde{\pi}$, marginalizing~\eqref{eq_aug_mh_ratio} over $\tau$ (with $\tau$ in the batches suppressed for clarity),
\begin{align}
	\tilde{\pi}(\btheta) 
	& \propto \left( \pi(\btheta)\right)^{1/T} \binom{n}{m}^{-1} \sum_{I\in\sI_m} e^{c_n (\hat{\mu}_I(\btheta) -  \mu(\btheta))}	
	\label{eq_tilde_pi}
\end{align}
where $T=n/c_n$ is the temperature. In this sense, the mini-batch stationary distribution is approximately a tempered version of the posterior, up to a bias term. Unlike pseudo-marginal MCMCs, we do not require constructing an unbiased estimate of the likelihood, which leads to improved computational efficiency. The bias becomes small (i.e. the bias term becomes close to 1) as $n$ increases for appropriate $m$ and $c_n$ since $\hat{\mu}_I(\btheta)-\mu(\btheta)$ becomes small. $c_n$ controls the trade-off between approximation error and the tempering amount -- a smaller $c_n$ leads to a smaller error but a higher temperature. The choice of $c_n$ and the exact error rate will also be discussed in Section~\ref{subsec_theory}.

We summarize the mini-batch MH algorithm in Algorithm~\ref{alg_mini_mh} (with $\tau$ suppressed for simplicity).
\begin{algorithm}[H]
	\caption{MH MCMC with batch tempering (MHBT)}
	\begin{algorithmic}
		\State \textbf{Input}: data $\bx$, batch size $m$, constant $c_n$, proposal $q(\btheta\to\btheta')$, log likelihood $\ell(\btheta)$,  initial $\btheta_0$, $I_0$.
		\For {$t=0,1, \dots$} 
		\State Draw $\btheta'$ from $q(\btheta_t \to \btheta')$, an index set $ I'\in\sI_m$ randomly, and $u\sim \text{Unif}[0,1]$. 
		\State Compute acceptance probability $r=\min \left\{ 1, \frac{q(\btheta' \to \btheta_t) e^{c_n \hat{\mu}_{I'}(\btheta')} }{q(\btheta_t \to \btheta') e^{c_n \hat{\mu}_{I_t} (\btheta_t)}} \right\}$
		\If{$u < r$} 
		\State $\btheta_{t+1}=\btheta'$, $I_{t+1}=I'$,
		\Else 
		\State $\btheta_{t+1}=\btheta_{t}$, $I_{t+1}=I_t$.
		\EndIf
		\EndFor
	\end{algorithmic}
	\label{alg_mini_mh}
\end{algorithm}

\subsection{Preservation of local optima and convergence to tempered posterior}
\label{subsec_theory}

In this section, we analyze the properties of the stationary distribution $\tilde{\pi}(\btheta)$. In particular, we show the convergence rate of the bias term in~\eqref{eq_tilde_pi} in terms of the two tuning parameters $m$ and $c_n$. Throughout the rest of the paper, for two positive sequences $a_n$ and $b_n$, we use the notation $a_n \asymp b_n$ if for large enough $n$, $a_n\leq c_1 b_n$, $b_n\leq c_2 a_n$ for some constants $c_1, c_2$ not depending on $n$. $\|\cdot\|_1, \|\cdot\|_2$ denote the $\ell_1$, $\ell_2$ norm for vectors, and $\|\cdot\|_{op}$ denotes the operator norm of a matrix. $\lfloor a \rfloor$ is the greatest integer smaller than or equal to $a$. $a\vee b=\max\{a,b\}$. $\E_{\btheta^*}$ is the expectation taken over the data which is generated by the true parameter $\btheta^*$. 

Consider the regime where both $n$ and $m$ are large with $m\leq n$. We will also allow the dimension $d$ to grow at some suitable rate with respect to $n$. We assume the likelihood function $p(x | \btheta) := p_{\btheta}(x)$, $x\in\mathcal{X}$, belongs to a parametric family satisfying the following conditions.

\begin{assum}
	There exist a function $L$ and a vector of measurable function $\mathcal{T}$ such that $| \log p_{\btheta}(x)-\log p_{\btheta}(y)| \leq L(\btheta) \Vert \mathcal{T}(x)-\mathcal{T}(y) \Vert_1 $, $x,y\in\mathcal{X}$, with $L_0:=\sup_{\btheta\in\Theta} L(\btheta) < \infty$ and $\E_{\btheta^*} e^{\delta_1 \Vert \mathcal{T}(X) \Vert_1} < \infty$ for some $\delta_1 > 0$. 
	\label{assum_smooth_x}
\end{assum}

\begin{assum}
	There exists a measurable function $M$ such that $| \log p_{\btheta}(x)-\log p_{\btheta'}(x)| \leq M(x) \Vert \btheta-\btheta'\Vert_1$ for all $\btheta, \btheta'\in\Theta$ and $x\in\sX$. In addition, there exists $\delta_2>0$ such that $\E_{\btheta^*} e^{\delta_2 M(X)} < \infty$. 
	\label{assum_smooth_theta}
\end{assum}

The above assumptions are mild and require 
the log likelihood $\log p_{\btheta}(x)$ to be suitably smooth in both $\btheta$ and $x$. Unlike some pseudo-marginal MCMC algorithms \cite{andrieu2009, maclaurin2014}, we do not require the likelihood to be bounded. We show in Appendix~\ref{sec_supp_applications} that these assumptions can be carried over to a number of commonly used models in statistics and machine learning, such as mixtures of exponential family distributions, linear regression with random feature vectors, and classification tasks with fully connected neural networks (which include logistic regression as a special case). In the exponential family example, $\mathcal{T}(\cdot)$ and $M(\cdot)$ are in fact functions of the sufficient statistic. In the neural network example, the constant $L(\btheta)$ is related to the network complexity measure.

For large $n$, it suffices to consider the population log likelihood $\mu_{\btheta} := \E_{\btheta^*} \log p_{\btheta}(X)$. Let $\btheta_0$ be a stationary point of $\mu_{\btheta}$ such that it represents a well-separated local optimum in the following sense.

\begin{assum}
	$\mu_{\btheta}$ is twice continuously differentiable in $\btheta$. $\btheta_0\in Int(\Theta)$ and the Hessian of $\mu_{\btheta}$ at $\btheta_0$ has eigenvalues $\lambda_i (H_{\btheta_0}) <0$ for all $i=1, \dots,d$.  
	\label{assum_hessian}
\end{assum}  

Note that the assumption implies there exist $\epsilon_0, \delta_0 >0$ such that $\mu_{\btheta_0}-\mu_{\btheta} \geq \epsilon_0 \Vert \btheta-\btheta_0\Vert_2$ for all $\Vert \btheta-\btheta_0\Vert_2\leq \delta_0$. Then we have the next theorem showing $\tilde{\pi}(\btheta)$ approximately preserves any well-separated local optimum.
\begin{thm}
	Suppose $\btheta_0$ is a stationary point of $\mu_{\btheta}$ satisfying Assumption~\ref{assum_hessian}. For some $\alpha>0$, let $c_n\to\infty$ be a sequence such that $\frac{dc_n^{2+\alpha}\log c_n}{m} \to 0$, $t$ be a fixed constant with $t\in(0,1/2)$, and $\delta_n = \sqrt{\frac{3\log(1/(1-2t))}{\epsilon_0 c_n}}$. Then under Assumptions~\ref{assum_smooth_x}, \ref{assum_smooth_theta}, for large $n$,
	\begin{align}
		\sup_{\btheta\in R_n} \log \tilde{\pi}(\btheta) \leq \sup_{\btheta\in \mathcal{B}(\btheta_0; \delta_n)} \log \tilde{\pi}(\btheta) -  \log(1/(1-2t)),
	\end{align}
	with probability at least $1-\eta_n$. Here $R_n = \{\btheta\in\Theta: \delta_n<\Vert\btheta-\btheta_0\Vert_2<\delta_0 \}$, $\mathcal{B}(\btheta_0; \delta_n)=\{\btheta\in\Theta: \Vert \btheta-\btheta_0\Vert_2 <\delta_n\}$, and $\eta_n\asymp \frac{1}{t^2\lfloor\frac{n}{m}\rfloor c_n^{\alpha}}$.  
	\label{thm_local_opt}
\end{thm}
The theorem states that with high probability, the supremum of $\log \tilde{\pi}(\btheta)$ in the shrinking ball $\mathcal{B}(\btheta_0; \delta_n)$ is larger than any point in the surrounding region $R_n$ by a constant margin. This guarantees with high probability $\tilde{\pi}(\btheta)$ has a local optimum lying in a shrinking neighborhood centered at $\btheta_0$. The preservation of local optima is important for optimization tasks. 

We can further bound the distance between $\tilde{\pi}(\btheta)$ and the tempered posterior $\pi^{1/T}(\btheta)$ with one more assumption. 
\begin{assum}
	$\Theta$ is compact.   
	\label{assum_compact}
\end{assum}  

\begin{thm}
	Denote $\pi_T(\btheta) \propto \pi^{1/T}(\btheta)$ the tempered posterior. Under Assumptions~\ref{assum_smooth_x}, \ref{assum_smooth_theta} and~\ref{assum_compact}, for some $\alpha>0$, $\epsilon_n \to 0$ slower than $c_n^{-\alpha/2}$, $c_n\to \infty$ such that $\frac{dc_n^{2+\alpha}\log c_n}{m} \to 0$, $D_{KL}( \pi_T \Vert  \tilde{\pi} ) \leq \epsilon_n$ with probability at least $1-\eta'_n$, $\eta'_n \asymp  \frac{1}{\epsilon_n^2\lfloor\frac{n}{m}\rfloor c_n^{\alpha}}$, for large $n$. 
	\label{thm_kl_div}
\end{thm}

The proofs of the above theorems can be found in Appendix~\ref{sec_supp_mainproofs}. 

\begin{remark}
	\label{rem:rate}
	\mbox{}
	\begin{enumerate}
		\item 
		Both Theorems~\ref{thm_local_opt} and~\ref{thm_kl_div} require $\frac{dc_n^{2+\alpha}\log c_n}{m} \to 0$, meaning $c_n$ and $d$ need to go to infinity at a controlled rate. The convergence regime in both theorems covers a wide spectrum of batch size $m$, from $\omega(1)$ to $O(n)$.
		
		\item 
		For a given $m$, if $d$ is fixed, we can choose $c_n$ to be a value close to but smaller than $\sqrt{m}$ to make sure the temperature is not too high while the convergence holds. In Section~\ref{subsec_low_d}, we show using numerical experiments that the choice of $c_n$ is very robust in low dimensional models.
		
		\item 
		The convergence requirement has a linear dependence on $d$. If $m=n^{\gamma}$ for some fraction $\gamma$, $d$ can also go to infinity at the rate of $n$ raised to some fractional power. 
	\end{enumerate}
\end{remark}

\bk

\subsection{MHBT with stochastic gradient based proposal for neural networks}

In large-scale machine learning tasks such as training deep neural networks (DNN), the high dimensionality and complex nature of the loss function surface have posed significant challenges for designing an MCMC sampler that can i) efficiently navigate the high dimensional surface, ii) result in a reasonable acceptance probability in the MH test, and iii) be computationally feasible. Recent studies on stochastic gradient MCMC have demonstrated their potential in training DNNs \citep{chen2014, li2016, ye2017}. However, these methods are derived from continuous-time SDEs, and each discretization step introduces some error which ideally could be corrected with an MH acceptance test. Many of these methods require a shrinking learning rate in order to circumvent the MH test. In this section, we propose and analyze a stochastic gradient-based proposal with appropriate MH correction, which is computationally efficient for DNN applications.  


\subsubsection*{Proposal with Reversible Stochastic Gradient Langevin Dynamics (RSGLD)}

Our goal is to design a proposal function that can explore a complex high dimensional surface efficiently guided by gradient information. We will start by considering the proposal used in SGLD, which has been widely adopted in the literature for large-scale training tasks. Let $\hat{g}_I(\btheta) = \frac{1}{|I|} \sum_{i \in I} \nabla_{\btheta} \ell_i(\btheta)$ be the average gradient of mini-batch $I$, the proposal move for SGLD is given by
\begin{align}
	\btheta' = \btheta + \epsilon \hat{g}_I(\btheta) + \frac{\sqrt{2\epsilon}}{n} N(0,I_d),
	\label{eq_sgld}
\end{align}
where $\epsilon$ is the learning rate,  $N(0,I_d)$ is the iid Gaussian noise. Note that we have written the learning rate in a form that is consistent with the convention for SGD, so $\epsilon$ differs from the learning rate in the convention for SGLD by a factor of $n$. The original SGLD avoids the MH correction step since it is costly to compute using the full data. 

In practice, in addition to the computational efficiency issue, another difficulty arises from the acceptance probability as $d$ increases. Using~\eqref{eq_sgld} as the proposal in Algorithm~\ref{alg_mini_mh}, it can be treated as a mini-batch version of the MALA algorithm \cite{roberts1996} (and the more general Hamiltonian MCMC). It is known that in these full-batch algorithms, $\epsilon$ needs to scale like $d^{-\frac{1}{4}}n^{-1}$ to maintain a reasonable acceptance probability \citep{neal2011}. As an illustration, we consider using~\eqref{eq_sgld} as the proposal in Algorithm~\ref{alg_mini_mh} to sample from the $d$-dimensional Gaussian $N(0, I_d)$, where $d=1, 10, 10^2, 10^3$, and $n=10^4, m=1000, c_n=20$. In Figure~\ref{fig_accept_prob}(a), we computed the average acceptance probability for the first 2000 iterations initializing at the origin and then selected the largest learning rate $\epsilon$ with average acceptance probability at least 0.5 and 0.1. $\epsilon$ was chosen from a grid that scales like $d^{-\frac{1}{4}}n^{-1}$. As can be seen, $\epsilon$ quickly diminishes to below $10^{-7}$ when the dimension reaches $10^3$, if we still want to maintain a reasonable acceptance probability. Such a small learning rate results in very slow convergence and is therefore usually infeasible for practical use.

Our proposal, Reversible Stochastic Gradient Langevin Dynamics (RSGLD), is based on SGLD but enhances the acceptance probability by allowing the sampler to move in the direction of either ascending or descending gradient with an adjusted Gaussian noise. Using RSGLD as the proposal in Algorithm~\ref{alg_mini_mh} gives us a mini-batch MH algorithm that both utilizes gradient information and is computationally efficient. Our proposal modifies~\eqref{eq_sgld} in two ways: i) a coin flip decides whether the move will be in the positive or negative direction of the gradient. For convenience, we will henceforth refer to a move in the positive (or negative) gradient direction as a \textit{forward} (or \textit{backward}) step; ii) the backward step is coupled with a larger Gaussian noise.
The new state $\btheta'$ is sampled by
\begin{align}
	\btheta' = \begin{cases}
		& \btheta + \epsilon \hat{g}_I(\btheta) + \frac{\sqrt{2\epsilon}}{n} N(0,I_d), \quad \text{with probability $1/2$,}	\\
		&  \btheta - \epsilon \hat{g}_I(\btheta) + \frac{\sqrt{2\epsilon}}{n} \beta N(0,I_d), \quad \text{with probability $1/2$}
	\end{cases}
	\label{eq_proposal}
\end{align}
for some constant $\beta\geq1$. Denote this proposal $q_{I}(\btheta \to \btheta')$, then
\begin{align}
	q_{I}(\btheta \to \btheta') = \frac{1}{2} \phi \left( \btheta'-\btheta - \epsilon \hat{g}_I(\btheta) ; \frac{2\epsilon}{n^2} I_d \right) 	+  \frac{1}{2} \phi \left( \btheta'-\btheta + \epsilon \hat{g}_I(\btheta) ; \frac{2\epsilon \beta^2}{n^2} I_d \right),
\end{align}
where $\phi(\cdot; \Sigma)$ is the density of a multivariate Gaussian with zero mean and covariance matrix $\Sigma$. 

In Algorithm~\ref{alg_mini_mh}, the acceptance probability for moving from $(\btheta_t, I_t) \to (\btheta', I')$ becomes
\begin{align}
	\min \left\{ 1, \frac{q_{I'}(\btheta' \to \btheta_t) e^{c_n\hat{\mu}_{I'}(\btheta')}}{q_{I_t}(\btheta_t \to \btheta')e^{c_n\hat{\mu}_{I_t}(\btheta_t)}} \right\}.
	\label{eq_accept_prob_mas}
\end{align}
Similar to the argument in Section~\ref{subsec_minibatch_mh}, we can show using an auxiliary variable the above mini-batch MH algorithm is closely related to a tempered MCMC. We refer to Appendix~\ref{sec_supp_rsgld} for details. 

As an illustration to show both the backward step and its associated, enlarged Gaussian noise increase the acceptance probability, we used the same Gaussian setting as before (sampling from $N(0, I_d)$, where $d=10, 10^2, 10^3$, and $n=10^4, m=1000, c_n=20$) and tested $\beta=1$, which corresponds to only adding the backward move; and $\beta=2$, which increases the size of the Gaussian noise in the backward move. In Figure~\ref{fig_accept_prob}(b)-(d), we can see both adding the backward move and increasing the Gaussian noise significantly improve the acceptance probability, and the trend is consistent for different dimensions.


\begin{figure}[h!!]
	\centering
	\subfloat[]{\includegraphics[width = .48\textwidth]{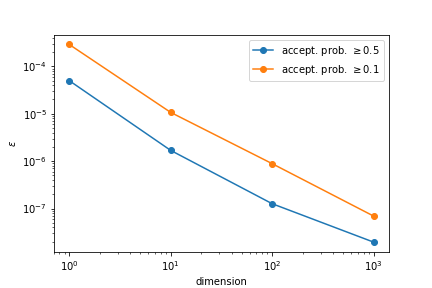}}
	\subfloat[]{\includegraphics[width = .48\textwidth]{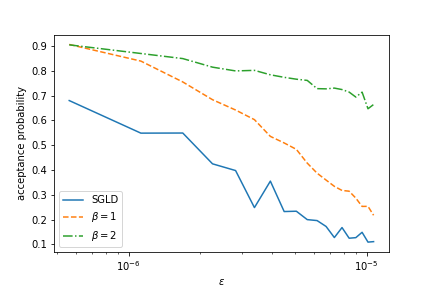}}\\
	\subfloat[]{\includegraphics[width = .48\textwidth]{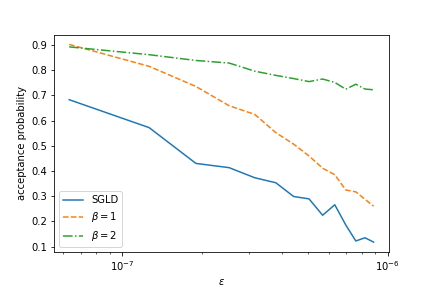}}
	\subfloat[]{\includegraphics[width = .48\textwidth]{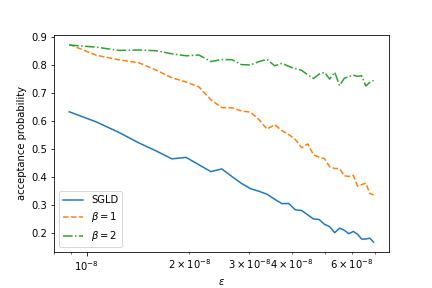}}
	\caption{(a) The largest $\epsilon$ allowed to achieve reasonable average acceptance probability on $N(0, I_d), d=1, 10, 10^2, 10^3$. (b), (c), (d), the average acceptance probability for SGLD, RSGLD ($\beta=1, 2$) for (b) $d=10$, (c) $d=10^2$, (d) $d=10^3$.}
	\label{fig_accept_prob}
\end{figure}

\subsubsection*{Analysis of acceptance probability}
In this section, we show that the RSGLD proposal leads to larger proposal ratio, thus increasing the MH ratio and acceptance probability overall. To focus on the behavior of the algorithm, we take the data $\bx$ as given and fixed, and the only randomness lies in the selection of data batch and the Gaussian perturbation. Let $\bZ\sim N(0, I_d)$ and $H_I(\btheta)$ be the Hessian matrix of $\hat{g}_I(\btheta)$ on mini-batch $I$. We assume the following conditions hold. 
\begin{assum}
	\label{assum_spec_norm}
	$\sup_{I\in\sI_m, \btheta} \Vert H_I(\btheta)\Vert_{op}\leq \lambda$ , where $\Vert \cdot \Vert_{op}$ is the operator norm. 
\end{assum}

\begin{assum}
	\label{assum_similar_grad}
	For every $\btheta$, all batches give similar gradients. More specifically, for any two batches $I$ and $J$, 
	\begin{align}
		\Vert \hat{g}_J(\btheta)-\hat{g}_I(\btheta) \Vert_2 = O(\epsilon \Vert \hat{g}_I(\btheta) \Vert_2) .
	\end{align} 
\end{assum}

\begin{proposition}
	\label{prop_accept_prob}
	For large $n$, suppose Assumptions~\ref{assum_spec_norm} and~\ref{assum_similar_grad} hold. Then depending on where the sampler is in the landscape of the target likelihood, we have the following approximations for the proposal ratio $\frac{q_{J}(\btheta' \to \btheta)}{q_{I}(\btheta \to \btheta')}$, where $\btheta$ is the current parameter value to be updated and $I$ is the current batch.
	
	\textbf{Case 1)}. Assume there exists a small constant $\eta_0$ such that $\Vert \bZ \Vert_2 \leq \eta_0 \cdot n \sqrt{\epsilon/2} \Vert \hat{g}_I(\btheta)\Vert_2 $ with high probability (i.e. with probability approaching 1), and the learning rate $\epsilon$ is small enough such that $\frac{n^2\epsilon^2(\epsilon\vee\eta_0)}{\beta^2-1} \Vert \hat{g}_I(\btheta) \Vert_2^2  = o(d)$ for large $d$, $\beta>1$. Then
	\begin{itemize}
		\item
		if the update in~\eqref{eq_proposal} results in a forward move, we have $\frac{q_{J}(\btheta' \to \btheta)}{q_{I}(\btheta \to \btheta')} > 1$ with high probability. 
		
		\item
		if the update in~\eqref{eq_proposal} results in a backward move, we have $\frac{q_{J}(\btheta' \to \btheta)}{q_{I}(\btheta \to \btheta')}=o_P(1)$. 
	\end{itemize}
	
	\textbf{Case 2)}. Assume $\Vert \hat{g}_I(\btheta) \Vert_2= 0$, and the learning rate $\epsilon$ is small enough such that $\epsilon = o(d^{-1})$ for large $d$. Then we have $\frac{q_{J}(\btheta' \to \btheta)}{q_{I}(\btheta \to \btheta')} = 1+o_P(1)$ for both directions in~\eqref{eq_proposal}. 
	
\end{proposition}
We defer the proof to Appendix~\ref{sec_supp_acceptprob}. 
\begin{remark}
	\mbox{}
	\begin{enumerate}
		\item 
		In this proposition, we consider the behavior of the proposal ratio in different regions of the landscape. The condition in Case 1) means the sampler is at a location where gradient information is strong. Simple rearranging in~\eqref{eq_proposal} shows in this case, the gradient part dominates the Gaussian noise.  In Case 2), the sampler has reached a flat region of the landscape.
		
		\item 
		If $\|\hat{g}_I(\btheta)\|_2=O(\sqrt{d})$, in Case 1) $\epsilon$ needs to satisfy $\epsilon\ll n^{-1} \sqrt{\beta^2-1}/\sqrt{\eta_0}$, the rate of which no longer depends on $d$ and scales better than before ($d^{-\frac{1}{4}}n^{-1}$). Sparse $\hat{g}_I(\btheta)$ (such as in typical neural networks) and large $\beta$ can allow for even larger learning rates.
		
		\item 
		The result in Case 1) implies it is more likely for the MH step to accept a forward move than a backward move when the gradient is strong. This is a desirable property in optimization tasks for maintaining efficiency. In particular, the proposal ratio is lower bounded by 1 in the forward direction and hence will no longer shrink the overall MH ratio to zero.  In Case 2), the proposal in the sampler behaves like a random walk if the learning rate is sufficiently small.
	\end{enumerate}
\end{remark}

\section{Experiments}
\label{sec:exp}
\subsection{Distributions in low dimensions}
\label{subsec_low_d}
\subsubsection*{Convergence to known posterior}

We first examined the convergence behavior of MHBT compared to the conventional MCMC sampler using the full dataset (termed full batch MCMC). As the analysis in Section~\ref{subsec_theory} suggests, MHBT converges to a tempered version of the original posterior distribution. In order to explicitly measure the distance from this posterior, we considered $d$-dimensional ($d=2$ and 5) Gaussian distributions with unknown mean $\btheta$, known covariance $I_d$, where the prior of $\btheta$ was set to be $N(0, I_d)$. We generated $n=10^5$ samples from this distribution with each true $\theta_i^*=2$. It follows then the posterior of $\btheta$ given the data $\bx$ is $N\left( \frac{n}{n+1}\bar{\bx}, (n+1)^{-1}I_d \right)$, where $\bar{\bx}$ is the sample average. Raising the posterior to temperature $T$ changes the variance to $\frac{T}{n+1} I_d $. Mini-batch sampling was performed with Algorithm~\ref{alg_mini_mh}, setting the proposal $q(\cdot)$ as a Gaussian random walk with step size $\delta$ and mini-batch size $m=1000$. Full batch MCMC was performed on the tempered posterior also with the same type of random walk proposal. The same step size $\delta$ was chosen for both algorithms and the average acceptance probability was around 0.3.

\begin{figure}[h!]
	\centering
	\subfloat[]{\includegraphics[width = .33\textwidth]{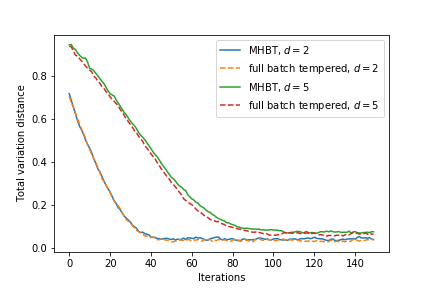}}
	\subfloat[]{\includegraphics[width=.33\textwidth]{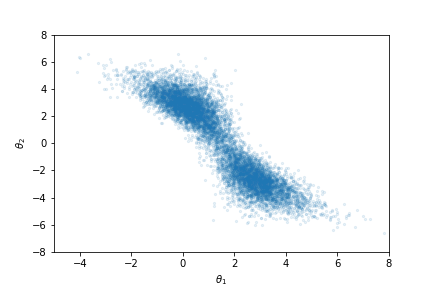}}
	\subfloat[]{\includegraphics[width=.33\textwidth]{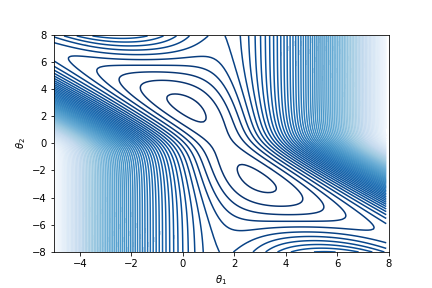}}\\
	\subfloat[]{\includegraphics[width = .33\textwidth]{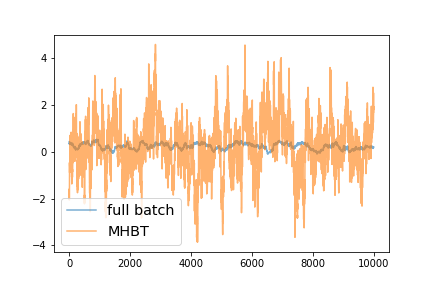}}
	\subfloat[]{\includegraphics[width = .33\textwidth]{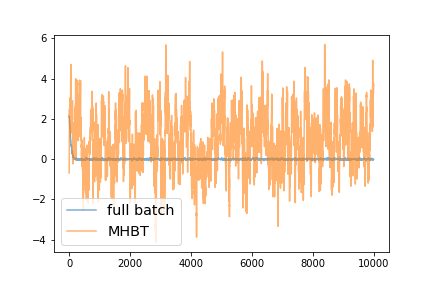}}
	\subfloat[]{\includegraphics[width = .33\textwidth]{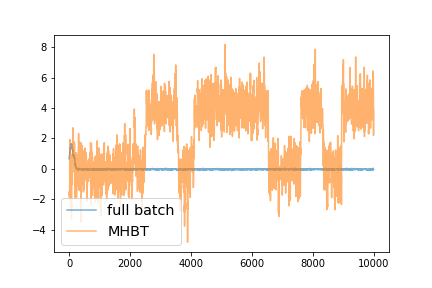}}
	\caption{(a) Total variation distance between the sampled distribution and true tempered posterior for $d$-dimensional Gaussian. $d=2,5$. (b), (c) Scatter plot of sampled $\btheta$ values vs. contour plot of the tempered log posterior; (d), (e), (f) trajectories of MHBT and full batch MCMC for the 2-component Gaussian mixture model with fixed $\theta_1=0$, and $\theta_2=0.5, 2, 4$ respectively.}
	\label{fig_gauss_meta}
\end{figure}

Figure~\ref{fig_gauss_meta}(a) shows the total variation distance between the sampled distributions and true tempered posterior for the two MCMC algorithms on $d$-dimensional Gaussian, as the number of iterations increases. The distance was calculated by running $10^5$ independent MCMC chains and taking the same number of independent samples from the tempered distribution, followed by discretization to group the values into $d$-dimensional histograms.
The results shown correspond to $c_n=20$, which is smaller than $\sqrt{m}$ as discussed in Remark~\ref{rem:rate}, although we note that a range of $c_n$ values (5-30) led to very similar results. For both $d=2$ and 5, MHBT converges at a rate almost identical to full batch MCMC to the tempered posterior.

\subsubsection*{Gaussian mixture}

To illustrate the tempering effect of MHBT and examine the accuracy of the approximation in Section~\ref{subsec_theory}, we consider an example in \cite{welling2011}.  We generated $n=10^5$ samples from a 2-component mixture Gaussian model with parameters $\btheta=(\theta_1, \theta_2)$ following:
\begin{align*}
	\btheta \sim N(0, \text{diag}(\sigma_1^2, \sigma_2^2)), \qquad X_i  \sim 0.5 N(\theta_1, \sigma_x^2) + 0.5 N(\theta_1+\theta_2, \sigma_x^2),
\end{align*}
where $\sigma_x^2=2$, $\sigma_1^2=10$,  $\sigma_2^2=1$. The posterior distribution of $\btheta$ given $\bx=(x_1, \dots, x_n)$ can be calculated explicitly as 
\begin{align*}
	\pi(\btheta) \propto & e^{-\frac{1}{2} \left( \theta_1^2/\sigma_1^2 + \theta_2^2/\sigma_2^2\right)} \prod_{i=1}^{n} \left( e^{-\frac{1}{4}(\theta_1^2-2\theta_1 x_i)} + e^{-\frac{1}{4}((\theta_1+\theta_2)^2-2(\theta_1+\theta_2) x_i)}  \right).
\end{align*}
We sampled $\btheta$ using Algorithm~\ref{alg_mini_mh}, where the proposal $q(\cdot)$ is the Gaussian random walk with step size $\delta$.  We set the mini-batch size $m$ to $1000$. There remain two tuning parameters in the algorithm: $c_n$ and $\delta$. We chose $c_n$ to be 20 and $\delta$ such that the average acceptance probability was around 0.3. Very similar results can be obtained by a range of $c_n$ values (e.g. 5-30).

Figure~\ref{fig_gauss_meta}(b)-(c) show the sampled $\btheta$ from $10^5$ iterations and the contour plot of the tempered log posterior,  $\log \pi_T(\btheta) \propto 1/T \log \pi(\btheta)$. We can see that the two modes in these plots coincide well. 

Figure~\ref{fig_gauss_meta}(d)-(f) compare the trajectory of MHBT with that of the full batch MCMC in one of the two dimensions. The latter sampling was performed on the original posterior distribution, and the step size of the random walk was chosen so that the average probability was around 0.3. We fixed $\theta_1=0$ and increased $\theta_2$ from 0.5 to 4 so that the two modes in the posterior distribution became increasingly separated. In each case, MHBT is capable of visiting the two modes of the distribution whereas the full batch MCMC is trapped in one of the modes. This highlights the effect of tempering brought about by the mini-batch algorithm, which makes the landscape smoother and easier for the sampler to travel.

\subsection{Neural networks}
\subsubsection*{Fully connected neural networks}

We first tested MHBT with RSGLD on the standard MNIST handwritten digit classification task. The dataset was loaded directly from TensorFlow tutorial and consists of 55,000 instances for training and 10,000 instances for testing. We considered a neural network containing one hidden layer with 600 nodes and ReLU activation function ($\sim4\times 10^5$ parameters). The outputs from the layer are connected to a 10-class softmax layer for classification. In this case, the log likelihood function is the negative of the cross entropy loss. The batch size was set to 100. We compared the performance of our method with a number of popular optimization methods in the neural network literature for a range of learning rates. In each training, we started RSGLD with a large $\beta$ to initiate the moves and gradually decreased it as the training progressed. 

\textbf{\textit{Choosing $\beta$.}}
Throughout training, we monitored the overall acceptance probability for each epoch, where by convention one epoch equals the total number of iterations it takes to step through the whole training dataset (in this case $55000/100=550$ iterations). We decreased $\beta$ according to the following adjustment phase once the acceptance probability became larger than 0.4 at the end of each epoch. During the adjustment phase, we ran 100 forward steps using the current parameter values and computed the MH acceptance probability. If the average probability of these forward steps exceeded 0.7, we decreased $\beta$ by 5\%. The maximum reduction allowed in each  adjustment phase was 50\%. The next epoch of training was then run with the new $\beta$ value. On the other hand, when the average probability for one epoch dropped below 0.2, we increased $\beta$ by 5\%. We observed that in all experiments, $\beta$ eventually stabilized to some constant slightly larger than 1. 

\textbf{\textit{Comparison with other methods.}}
We performed extensive comparison with SGD and SGLD using various learning rates and multiple rounds of training to assess the stability of each method. Each round of training lasted $2.75\times10^5$ iterations (500 epochs), and all the parameters were initialized with independent $N(0,0.03)$ distribution. The same batch size (100) was used for all the methods. In this high dimensional setting, we explored a range of $c_n$ values around the batch size and show results using $c_n=100$. We additionally tested $c_n=50, 200$ under the same settings; the results are very similar thus omitted. 

Table~\ref{tab_testing_error_mnist} shows the prediction errors of the three methods on the testing set, using the top class from the softmax layer as the predicted label. Each number is the median error obtained from 30 training rounds with the corresponding standard deviation shown in parentheses. Overall, the performance of RSGLD improves with large learning rate and eventually achieves better accuracy (smaller error) than that attainable by SGD or SGLD at any learning rate. RSGLD shows substantially better stability for large learning rate than the other two methods. In particular, when the learning rate is 0.2 or larger, SGD and SGLD can fail to converge completely for a significant fraction of the training rounds, which explains the large standard deviations. In general, the standard deviation of errors increases with the learning rate for all the methods, showing stability is hard to achieve with a large learning rate although it can lead to faster convergence and potentially better prediction. As explained in \cite{xing2018}, using a large learning rate can help algorithms maintain a trajectory high from the valley floor and more easily overcome energy barriers as they explore the loss surface with stochastic gradients. In this sense, the stability of RSGLD under large learning rates is beneficial for training DNNs. We also observe that in all the experiments, the backward step in RSGLD was much less likely to be accepted compared to the forward step, which is discussed in Case 1) of Proposition~\ref{prop_accept_prob} and is desirable for optimization efficiency. Since the forward step is identical to SGLD, this suggests a main reason for improvement offered by RSGLD lies in the algorithm being able to select a more efficient trajectory through the parameter space via the MH correction step. 

\begin{table*}[ht]
	\centering 
	\begin{tabular}{c @{\hspace{0.5\tabcolsep}} c @{\hspace{0.5\tabcolsep}} c @{\hspace{0.5\tabcolsep}} c @{\hspace{0.5\tabcolsep}} c @{\hspace{0.5\tabcolsep}} c }
		& \multicolumn{5}{c}{Top class prediction error (\%) on the testing set}	\\
		\hline\hline
		$\epsilon$ & 0.01 & 0.02 & 0.05 & 0.08 & 0.1 \\
		\hline
		RSGLD & $2.01 (0.03)$ & $1.82 (0.03)$ & $1.72 (0.03)$ & $1.75 (0.03)$ & $1.73 (0.04)$ \\
		SGD	 & $1.81 (0.02)$ & $1.78 (0.02)$ & $1.73 (0.03)$ & $1.75 (0.03)$ & $1.75 (0.06)$ \\
		SGLD & $1.81 (0.02)$ & $1.78 (0.02)$ & $1.73 (0.03)$  & $1.72 (0.03)$ & $1.75 (0.07)$ \\
		\hline
		& 0.2 & 0.3 & 0.4 & 0.5 & 0.6 \\
		RSGLD & $1.75 (0.13)$ & $1.7 (0.18)$ &	$1.68 (0.33)$ & $1.66 (11.4)$ & $1.71 (26.9)$ \\
		SGD &  $1.76 (16.1)$ & $1.85 (42.1)$ & $89.7 (44.4)$ & $89.7 (33.4)$ &	$89.8 (24.0)$	\\
		SGLD & $1.8 (33.3)$ & $1.84 (42.1)$ & $89.7 (43.3)$ & $89.9 (26.9)$ & $89.8 (37.9)$
	\end{tabular}
	\caption{MNIST top class prediction error (\%) on the testing set using 30 training rounds for each learning rate. Each number is the median error with standard deviation in parentheses.}
	\label{tab_testing_error_mnist}
\end{table*}

In addition to checking the average performance of the methods from multiple training rounds, we also examine the lowest prediction error achieved under each learning rate from 30 rounds of training. Since SGD and SGLD did not converge most of the time under large learning rates, showing the average or median error would make the plot scale badly. Fig~\ref{fig_mnist_error}(a) shows a trend similar to Table~\ref{tab_testing_error_mnist} with RSGLD outperforming the other two methods for large learning rates. Overall the lowest error is achieved by RSGLD with learning rate around 0.4-0.5. Examples of detailed testing error trajectories for various methods are shown in Fig~\ref{fig_mnist_error}(b), where for each method we selected the learning rate with the best performance. We have further included RMSprop \cite{tieleman2012} with learning rate 0.005 and Adam \cite{kingma2014} with learning rate 0.001 for comparison. The learning rate was chosen by optimization via grid search for these two methods.

\begin{figure}[h!!]
	\centering
	\subfloat[]{\includegraphics[width = .48\textwidth]{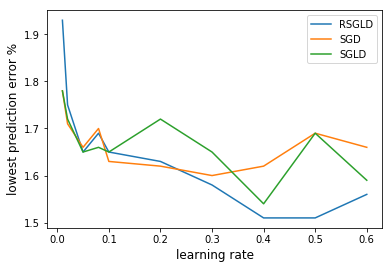}}
	\subfloat[]{\includegraphics[width = .48\textwidth]{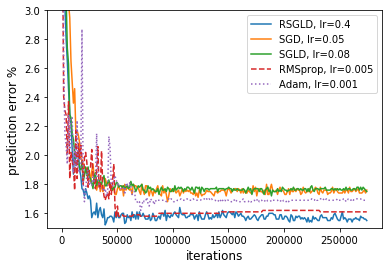}}
	\caption{(a) Lowest error rate in \% achieved by the three methods out of 30 training rounds using various learning rates.(b) Examples of testing error trajectories using different training methods.}
	\label{fig_mnist_error}
\end{figure}

\subsubsection*{Convolutional neural networks (CNN)}
We next tested a standard three-layer CNN on the CIFAR-10 RGB image dataset \citep{krizhevsky2009}, the detailed architecture of which is listed in Appendix Table~\ref{tab_cnn_structure_cifar10}. The network has around $4\times10^6$ parameters. The dataset consists of 60000 $32\times32$ RGB images in 10 classes, with 50000 for training and 10000 for testing. All parameters were initialized independently with $N(0,0.02)$ distribution. The same batch size and $c_n$ were used, and the same schedule was used for decreasing $\beta$ as in the last example. Similar to the comparison performed on MNIST, we used 20 rounds of independent training for each learning rate to check the accuracy and stability of RSGLD, SGD and SGLD, with each round lasting for $10^5$ iterations. As shown in Table~\ref{tab_testing_error_cifar10}, RSGLD consistently outperformed the other two methods and the margin of difference becomes larger as the learning rate increases. 

\begin{table}[h!]
	\centering 
	\begin{tabular}{c @{\hspace{1\tabcolsep}} c @{\hspace{1\tabcolsep}} c @{\hspace{1\tabcolsep}}  c@{\hspace{1\tabcolsep}}  c @{\hspace{1\tabcolsep}} }
		& \multicolumn{4}{c}{Top class prediction error (\%) on the testing set}	\\
		\hline\hline
		$\epsilon$ & 0.005 & 0.008 & 0.02 & 0.04  \\
		\hline
		RSGLD & 26.93/26.29 & 26.81/26.01 & 26.95/26.36 & 27.34/26.75 	\\
		SGD	 & 27.03/26.55 & 27.03/26.43 & 27.27/26.85 & 27.85/27.14 	\\
		SGLD & 27.00/26.60 & 26.88/26.19 & 27.31/26.70 & 27.85/27.08 	\\
	\end{tabular}
	\caption{CIFAR-10 top class prediction error (\%) on the testing set using 20 training rounds for each learning rate. The numbers shown are the median/lowest errors out of 20 rounds.}
	\label{tab_testing_error_cifar10}
\end{table}

\section{Conclusion}

In this paper, we study an efficient MH-MCMC algorithm which uses mini-batches of data. We draw connections between the stationary distribution of this Markov chain and the tempered posterior, and provide the approximation errors for a general class of likelihood functions. We also propose RSGLD, a stochastic gradient based proposal to help the sampler navigate complex high dimensional surface with reasonable acceptance probability in the MH acceptance test. Empirically, we demonstrate the algorithm has good convergence behavior and the tempering effect helps move between well separated modes in classical low dimensional models. We demonstrate the efficacy of RSGLD in training neural networks with the MNIST and CIFAR-10 datasets and show that compared to popular optimization methods, we achieve improved accuracy and stability when the learning rate is large. 

\appendixtitleon
\appendixtitletocon
\begin{appendices}
\section{Proofs of the main theorems}
\label{sec_supp_mainproofs}
In this section, we first prove Theorems~\ref{thm_local_opt} and~\ref{thm_kl_div} in the main paper. 
We start the analysis by first showing two concentration lemmas. For brevity, we will write $\E_{\btheta^*}(\cdot)$ as $\E(\cdot)$. $C, C_1, \dots$ are general constants and might be different in every appearance.

\begin{lem}
Let $c_n$ be a sequence going to infinity such that $c_n^{2}/m \to 0$ and $\Omega\subset\Theta$ be a bounded subset. Under Assumptions~\ref{assum_smooth_x}-\ref{assum_smooth_theta}, 
\begin{align}
P\left( \sup_{\btheta \in \Omega} c_n |\hat{\mu}(\btheta)-\mu_{\btheta} | > t\right) \leq & \begin{cases}
r_n^{-d} e^{-\frac{mt^2}{16C_1(\btheta^*,\delta_1)c_n^2}} \quad & \text{ if } d_n \leq t\leq 2\delta_1 C_1(\btheta^*, \delta_1) c_n/L_0,	\\
r_n^{-d} e^{-\frac{\delta_1 mt}{8c_n L_0}} \quad & \text{ if } t > 2\delta_1 C_1(\btheta^*, \delta_1) c_n/L_0. 
\end{cases}
\label{eq_uniform_theta}
\end{align}
where $C_1$ is a constant depending only on $\btheta^*$ and $\delta_1$, $r_n\asymp c_n^{-1}d_n$, $d_n \to 0$.
\label{lem_unif_chernoff}
\end{lem}

\begin{proof}
We first consider fixed $\btheta$. Let $\tilde{\ell}_{\btheta} (X_i) = \log p_{\btheta}(X_i) - \mu_{\btheta}$ and $(Y_1, \dots, Y_n)$ be an independent copy of $(X_1, \dots, X_n)$, then for $s>0$,
\begin{align}
\E e^{s\sum_{i=1}^{m} \tilde{\ell}_{\btheta} (X_i)} & \leq \E e^{s\sum_{i=1}^{m} (\tilde{\ell}_{\btheta} (X_i)-\tilde{\ell}_{\btheta} (Y_i))} 	\notag\\
& = \E e^{s\sum_{i=1}^{m} (\log p_{\btheta} (X_i)-\log p_{\btheta} (Y_i))} 
\end{align}
since $\E e^{-s\sum_{i=1}^{m} \tilde{\ell}_{\btheta} (Y_i)} \geq 1$. By Assumption 2,  
\begin{align}
|\log p_{\btheta} (X_i)-\log p_{\btheta} (Y_i) | \leq L_0 \Vert \mathcal{T}(X_i) - \mathcal{T}(Y_i) \Vert_1.
\label{eq_lip_w}
\end{align}
Letting $\mathcal{W}_i = \log p_{\btheta} (X_i)-\log p_{\btheta} (Y_i)$, $g(x) = \frac{2(e^x-1-x)}{x^2}$, we have 
\begin{align*}
\E e^{s \mathcal{W}_i}  =  & \E \left( \sum_{j=0}^{\infty} \frac{s^j \mathcal{W}_i^j}{j!}\right)	\\
 = & 1+\frac{s^2}{2} \E\left( \mathcal{W}_i^2 g(s\mathcal{W}_i) \right) \leq e^{\frac{s^2}{2} \E\left( \mathcal{W}_i^2 g(s\mathcal{W}_i) \right)},
\end{align*}
where 
\begin{align*}
\E\left( \mathcal{W}_i^2 g(s\mathcal{W}_i) \right) & \leq \E\left( \mathcal{W}_i^2 g(s |\mathcal{W}_i|) \right) \leq \E\left( \mathcal{W}_i^2 g(\delta_1 |\mathcal{W}_i|/L_0) \right)		\\
& = \frac{2 L_0^2}{\delta_1^2}\E\left( e^{\delta_1|\mathcal{W}_i|/L_0} - \delta_1|\mathcal{W}_i|/L_0 -1 \right)		\\
& \leq \frac{2 L^2_0}{\delta_1^2}\E\left( e^{\delta_1 \Vert \mathcal{T}(X_i) - \mathcal{T}(Y_i)\Vert_1}-1-\delta_1 \Vert \mathcal{T}(X_i) - \mathcal{T}(Y_i)\Vert_1 \right)	\\
& \leq C_1(\btheta^*, \delta_1)
\end{align*}
for $0<s\leq \delta_1/L_0$.  Putting all the parts together, 
\begin{align}
\E e^{s \sum_{i=1}^{m} \tilde{\ell}_{\btheta} (X_i) } \leq e^{\frac{s^2mC_1(\btheta^*, \delta_1)}{2}}.
\label{eq_exp_memoment}
\end{align}

To show uniform concentration, consider a $\epsilon_n$-covering of the set $\Omega$ with centers $\{\btheta_1, \dots, \btheta_N\}$, where $N=K (\epsilon_n)^{-d} diam(\Omega)^d$ for some constant $K$ since $\Omega$ is bounded. For any $\btheta \in \mathcal{B}_1(\btheta_j, \epsilon_n)$, where $\mathcal{B}_1$ denotes the $\ell_1$ ball, 
\begin{align}
\sup_{\btheta \in \mathcal{B}_1(\btheta_j, \epsilon_n)} | \log p_{\btheta} (X_i) - \log p_{\btheta_j}(X_i) | & \leq \epsilon_n M(X_i)
\end{align}
by Assumption~\ref{assum_smooth_theta}, and
\begin{align}
\sup_{\btheta \in \mathcal{B}_1(\btheta_j, \epsilon_n)} \left\vert \tilde{\ell}_{\btheta}(X_i) - \tilde{\ell}_{\btheta_j}(X_i) \right\vert \leq \epsilon_n \left(M(X_i)+\E(M(X_i)) \right).
\label{eq_covering_bound}
\end{align}  
It follows then 
\begin{align}
& \E e^{2s\sup_{\btheta \in \mathcal{B}_1(\btheta_j, \epsilon_n)} |\tilde{\ell}_{\btheta}(X_i) - \tilde{\ell}_{\btheta_j}(X_i)|} 	\notag\\
 \leq & \E e^{2s \epsilon_n \left(M(X_i)+\E(M(X_i)) \right)} 	\notag\\
 \leq & \E e^{4s \epsilon_n M(X_i)} \leq \left( \E e^{\delta_2 M(X_i)}\right)^{\frac{4s\epsilon_n}{\delta_2}}=e^{C_2(\btheta^*, \delta_2) \cdot \frac{4s\epsilon_n}{\delta_2}}
\label{eq_lip_exp_moment}
\end{align}
for $4s\epsilon_n < \delta_2$, again by Assumption 3. Next note 

\begin{align}
& P\left( \sup_{\btheta \in \mathcal{B}_1(\btheta_j, \epsilon_n)}  c_n (\hat{\mu}(\btheta)-\mu_{\btheta} ) > t \right) 	\notag\\
\leq & P \left( \frac{c_n}{m} \sum_{i=1}^{m} \sup_{\btheta \in  \mathcal{B}_1(\btheta_j, \epsilon_n)} \left\vert \tilde{\ell}_{\btheta}(X_i) - \tilde{\ell}_{\btheta_j}(X_i) \right\vert  + \tilde{\ell}_{\btheta_j}(X_i) > t \right)	\notag\\
\leq & \E \exp \left\{ s \sum_{i=1}^{m} \sup_{\btheta \in  \mathcal{B}_1(\btheta_j, \epsilon_n)} \left\vert \tilde{\ell}_{\btheta}(X_i) - \tilde{\ell}_{\btheta_j}(X_i) \right\vert  + \tilde{\ell}_{\btheta_j}(X_i) \right\} e^{-sm c_n^{-1} t}	\notag\\
\leq & \left( \E e^{2s\sum_{i=1}^{m} \sup_{\btheta \in  \mathcal{B}_1(\btheta_j, \epsilon_n)} |\tilde{\ell}_{\btheta}(X_i) - \tilde{\ell}_{\btheta_j}(X_i)|} \right)^{1/2} \left( \E e^{2s \sum_{i=1}^{m} \tilde{\ell}_{\btheta_j}(X_i)} \right)^{1/2} e^{-smc_n^{-1}t}	\notag\\
\leq & e^{s^2 m C_1(\btheta^*, \delta) + 2s m \epsilon_n C_2(\btheta^*, \delta_2)/\delta_2 -sm c_n^{-1} t}	\notag\\
\leq & e^{s^2 m C_1(\btheta^*, \delta) - smc_n^{-1} t/2}
\end{align}
using the same calculation as in~\eqref{eq_exp_memoment}, for $2s<\delta_1/L_0$, $\epsilon_n = c_n^{-1}d_n \delta_2/(4C_2)\asymp c_n^{-1}d_n$, $t\ \geq d_n$. For $d_n\leq t\leq 2\delta_1 C_1(\btheta^*, \delta_1) c_n/L_0$, the bound is minimized at $s=\frac{t}{4c_n C_1(\btheta^*, \delta_1)}$ with value $e^{-\frac{mt^2}{16C_1(\btheta^*, \delta_1)c_n^2}}$. For $t > 2\delta_1 C_1(\btheta^*, \delta_1) c_n/L_0$, set $s=\delta_1/(2L_0)$, and  

\begin{align*}
e^{\delta_1^2 mC_1(\btheta^*, \delta_1)/(4L_0^2) - \delta_1 m c_n^{-1}t/(4L_0) } & \leq e^{m\left( \delta_1 c_n^{-1}t/(8L_0)- \delta_1 c_n^{-1}t/(4L_0)\right)}	\\
& = e^{-\frac{\delta_1 m t}{8 c_n L_0}}.
\end{align*}

Now with the $\epsilon_n$-covering,
\begin{align}
& P\left( \sup_{\btheta\in\Omega} c_n |\hat{\mu}(\btheta)-\mu_{\btheta} | > t  \right)	 \notag\\
\leq & \sum_{j=1}^{N} P \left( \sup_{\btheta \in \mathcal{B}_1(\btheta_{j}, \epsilon_n)} c_n |\hat{\mu}(\btheta)-\mu_{\btheta} | > t \right)	\notag\\
\leq & \begin{cases}
 r_n^{-d} e^{-\frac{mt^2}{16C_1(\btheta^*,\delta_1)c_n^2}} \quad & \text{ if } d_n\leq t\leq 2\delta_1 C_1(\btheta^*, \delta_1) c_n/L_0,	\\
r_n^{-d} e^{-\frac{\delta_1 mt}{8c_n L_0}} \quad & \text{ if } t > 2\delta_1 C_1(\btheta^*, \delta_1) c_n/L_0, 
\end{cases}
\end{align}
where $r_n\asymp c_n^{-1}d_n$.
\end{proof}

We can now provide a uniform bound for the term $U(\btheta):=\binom{n}{m}^{-1} \sum_{I\in\sI_m} e^{c_n(\hat{\mu}_I(\btheta)-\mu_{\btheta})}$. 
\begin{lem}
For some $\alpha>0$, let $c_n$ be a sequence going to infinity at a rate such that $dc_n^{2+\alpha}\log c_n /m\to 0$. For $\btheta\in \Omega$, $\Omega$ being a compact subset of $\Theta$, 
\begin{align}
P\left( \sup_{\btheta\in \Omega} \left\vert U(\btheta)-1 \right\vert > t + u_n \right) \leq \frac{C}{t^2\lfloor\frac{n}{m}\rfloor c_n^{\alpha}}
\end{align}
for any fixed $t>0$, and $u_n \asymp  c_n^{-\alpha/2}$. 
\label{lem_u_stat}
\end{lem}

\begin{proof}
First note that $\sup_{\btheta\in \Omega}$ $U(\btheta)$ is bounded above by  
\[
\bar{U} = \binom{n}{m}^{-1} \sum_{I\in\sI_m}\sup_{\btheta\in \Omega} g_{\btheta}(X_{I(1)}, \dots, X_{I(m)}) 
\]
with $g_{\btheta}(X_{I(1)}, \dots, X_{I(m)})  = e^{c_n (\hat{\mu}_I(\btheta)-\mu_{\btheta})}$. Thus $\bar{U}$ is a U-statistic. We first provide a bound on its expectation, 
\begin{align}
 & \left\vert \E(\sup_{\btheta\in \Omega} g_{\btheta}(X_{1}, \dots, X_{m}) -1 \right\vert 	\notag\\
\leq & \E \left\vert \sup_{\btheta\in \Omega} g_{\btheta}(X_{1}, \dots, X_{m}) -1 \right\vert 	\notag\\
= & \int_{0}^{\infty} P\left( \left\vert \sup_{\btheta\in \Omega} g_{\btheta}(X_{1}, \dots, X_{m}) -1 \right\vert > s\right)ds	\notag\\
\leq & \int_{0}^{\infty} P\left( \sup_{\btheta\in\Omega} c_n(\hat{\mu}(\btheta)-\mu_{\btheta}) > \log(1+s) \right) ds 		\notag\\
& \quad +  \int_{0}^{1} P\left( \sup_{\btheta\in\Omega} c_n(\hat{\mu}(\btheta)-\mu_{\btheta}) < \log(1-s) \right) ds		\notag\\
\leq & \int_{d_n}^{\infty} P\left( \sup_{\btheta\in\Omega} c_n | \hat{\mu}(\btheta)-\mu_{\btheta} | > \log(1+s) \right) ds 		\notag\\
& \quad +  \int_{d_n}^{1} P\left( \sup_{\btheta\in\Omega} c_n |\hat{\mu}(\btheta)-\mu_{\btheta}| > -\log(1-s) \right) ds + 2d_n.
\label{eq_expec_bound}
\end{align}
Since $\log(1+d_n)\asymp d_n$, by Lemma~\ref{lem_unif_chernoff},  
\begin{align}
 & \int_{d_n}^{\infty} P\left( \sup_{\btheta\in\Omega} c_n | \hat{\mu}(\btheta)-\mu_{\btheta} | > \log(1+s) \right) ds 	\notag\\
 \leq &   r_n^{-d} \int_{d_n}^{e^{C_2c_n}-1} e^{-C_1m\log ^2(1+s)/ c_n^2} ds + r_n^{-d} \int_{e^{C_2c_n}-1}^{\infty} e^{-C_3 m\log(1+s)/c_n} ds 	\notag\\
 = & r_n^{-d} e^{\frac{c_n^2}{4C_1m}} \int_{\log(1+d_n)}^{C_2c_n} \exp\left(-\frac{C_1m}{c_n^2}(u-\frac{c_n^2}{2C_1m})^2\right) du + r_n^{-d} \int_{e^{C_2c_n}-1}^{\infty}(1+s)^{-\frac{C_3m}{c_n}} ds	 	\notag\\
 \leq & C c_n^{(1+\alpha/2)d}\left( \exp\left(-\frac{C_1m}{c_n^{2+\alpha}}\right)c_n + \frac{c_n}{m}\exp(-C_2m) \right) \notag\\
 \leq & C  \exp\left(-\frac{C_1m}{c_n^{{2+\alpha}}} \right),
 \label{eq_integral_bound}
\end{align}
taking $d_n=c_n^{-\alpha/2}$, provided $\frac{m}{dc_n^{2+\alpha}\log c_n} \to \infty$.
The same rate can be obtained for the second term in~\eqref{eq_expec_bound}. Overall we have
\begin{align}
| \E(\sup_{\btheta\in \Omega }g_{\btheta}(X_{1}, \dots, X_{m})) -1 | \leq u_n. 
\label{eq_mean_bound}
\end{align}
$u_n \asymp  \exp\left(-\frac{C_1m}{c_n^{2+\alpha}} \right) \vee c_n^{-\alpha/2}\asymp c_n^{-\alpha/2}$.

Next we derive the concentration of the U-statistic $\bar{U}$ around its expectation. Let $$\tilde{g}_{\btheta}(X_1, \dots, X_m) =\sup_{\btheta\in \Omega}g_{\btheta}(X_1, \dots, X_m)- \E(\sup_{\btheta\in \Omega} g_{\btheta}(X_1, \dots, X_m)).$$ Noting the symmetry of $g_{\btheta}$, we can first rewrite $\bar{U}-\E(\bar{U})$ as 
\begin{align}
\bar{U}-\E(\bar{U}) & =  \frac{1}{n!} \sum_{i_1, \dots, i_n} V(X_{i_1}, \dots, X_{i_n}), 
\end{align}
where $\{i_1, \dots, i_n\}$ is a permutation of $\{1, \dots, n\}$ and 
\begin{align*}
V(X_{i_1}, \dots, X_{i_n}) = \frac{1}{n_m} \sum_{k=0}^{n_m-1} \tilde{g}_{\btheta} (X_{i_{km+1}}, \dots, X_{i_{(k+1)m}})
\end{align*}
for $n_m= \lfloor\frac{n}{m}\rfloor$. Then for any fixed $t>0$,
\begin{align}
	P(|\bar{U}-\E(\bar{U})|>t) & \leq \frac{1}{n_m^2 t^2} \E \left( \frac{1}{n!} \sum_{i_1, \dots, i_n} \sum_{k=0}^{n_m-1} \tilde{g}_{\btheta} (X_{i_{km+1}}, \dots, X_{i_{(k+1)m}}) \right)^2 	\notag\\	
	& \leq \frac{1}{n_m^2 t^2 n!}  \sum_{i_1, \dots, i_n} \E \left(\sum_{k=0}^{n_m-1} \tilde{g}_{\btheta} (X_{i_{km+1}}, \dots, X_{i_{(k+1)m}}) \right)^2 	\notag\\
	& =  \frac{1}{n_m t^2}  \E \tilde{g}^2_{\btheta} (X_1, \dots, X_m). 
	\label{eq_u_stat_moment}
\end{align}
It remains to calculate the second moment of $\tilde{g}^2_{\btheta}$. Using~\eqref{eq_mean_bound}, 
\begin{align*}
	\E \tilde{g}^2_{\btheta} (X_1, \dots, X_m) & = \int_0^{\infty} P(\tilde{g}^2_{\btheta} (X_1, \dots, X_m) \geq x) dx	\notag\\
	& \leq \int_0^{\infty} P\left((\sup_{\btheta\in \Omega}g_{\btheta}(X_1, \dots, X_m)-1)^2 \geq \frac{x}{4} \right)dx + 4u_n^2	
	\notag\\
	& \leq \int_{4d_n^2}^{\infty} P\left( \sup_{\btheta\in\Omega} c_n |\hat{\mu}(\btheta)-\mu_{\btheta}| \geq \log(\sqrt{x/4}+1) \right) dx	\notag\\
	& \qquad +  \int_{4d_n^2}^4 P\left(  \sup_{\btheta\in\Omega} c_n |\hat{\mu}(\btheta)-\mu_{\btheta}| \geq -\log(1-\sqrt{x/4})\right)dx + C d_n^2 + 4u_n^2.
\end{align*}
By Lemma~\ref{lem_unif_chernoff}, the first integral is bounded by
\begin{align*}
 &  r_n^{-d} \int_{4d_n^2}^{4e^{2C_1c_n} } e^{-C_1m\log ^2(1+\sqrt{x/4})/ c_n^2} dx	+ r_n^{-d} \int_{4e^{2C_1c_n} }^{\infty} e^{-C_3 m\log(1+\sqrt{x/4})/c_n} dx  \\
 \leq &  C c_n^{(1+\alpha/2)d}\left( \exp\left(-\frac{C_1m}{c_n^{2+\alpha}}\right)c_n + \frac{c_n}{m}\exp(-C_2m) \right)		\\
 \leq & C  \exp\left(-\frac{C_1m}{c_n^{{2+\alpha}}} \right),
\end{align*}
using a similar calculation as~\eqref{eq_integral_bound}, taking $d_n=c_n^{-\alpha/2}$, provided $\frac{m}{dc_n^{2+\alpha}\log c_n} \to \infty$. The second integral can be calculated in the same way to obtain the same order. Thus
\begin{align}
	\E \tilde{g}^2_{\btheta} (X_1, \dots, X_m) & \leq C \cdot \exp\left(-\frac{C_1m}{c_n^{2+\alpha}}\right) \vee c_n^{-\alpha}\asymp c_n^{-\alpha}.
	\label{eq_second_moment_bd}
\end{align}
\eqref{eq_u_stat_moment} and \eqref{eq_second_moment_bd} imply
\begin{align*}
P(|\bar{U}-\E(\bar{U})|>t) & \leq \frac{C}{n_m t^2c_n^{\alpha}}.
\end{align*}
Together with~\eqref{eq_mean_bound}, we obtain the required bound in one direction. 

The proof for the other direction is similar noting $ \inf_{\btheta\in \Omega} U(\btheta) \geq \underline{U}$, where
\[
\underline{U} = \binom{n}{m}^{-1} \sum_{I\in\sI_m} \inf_{\btheta\in \Omega} g_{\btheta}(X_{I(1)}, \dots, X_{I(m)}). 
\]
\end{proof}

Now we are ready to prove Theorem 1. 

\begin{proof}[Proof of Theorem 1:]

It follows from Equation (4) in the paper that 
\begin{align*}
\log\tilde{\pi}(\btheta) = c_n \mu_{\btheta} + \log U(\btheta) + C(x)
\end{align*}
for some normalizing constant $C(x)$. 
Observe that maximizing $\log\tilde{\pi}(\btheta)$  is equivalent to maximizing
 \[
S_n(\btheta) = \mu_{\btheta} + c_n^{-1} \log U(\btheta).
\]
Assumption~\ref{assum_hessian} implies there exist $\epsilon_0, \delta_0 > 0$ such that 
\begin{align}
\mu_{\btheta_0} - \mu_{\btheta} \geq \epsilon_0 \Vert \btheta-\btheta_0 \Vert_2
\label{eq_bound_pop_llh}
\end{align}
for all $\btheta\in \mathcal{B}(\btheta_0; \delta_0)$, so the local optimum is well separated.

Lemma~\ref{lem_u_stat} shows $\log U(\btheta)$ is uniformly small in $\btheta$ for $\btheta\in \mathcal{B}(\btheta_0;\delta_0)$. Taking fixed $t$ and $t<1/2$, for large enough $n$,
\begin{align}
	\sup_{\btheta\in \mathcal{B}(\btheta_0;\delta_0)} |
	\log U(\btheta)| & \leq  \log (1+2t) \vee \log(1/(1-2t))	\notag\\
	& \leq \log(1/(1-2t))
\end{align}
with probability at least $1-\eta_n$.
Now we have 
\begin{align}
\sup_{\btheta\in \mathcal{B}(\btheta_0;\delta_0)} ( S_n(\btheta)- \mu_{\btheta} ) & = \frac{1}{c_n} \sup_{\btheta\in \mathcal{B}(\btheta_0;\delta_0)} \log ( U(\btheta) )	\notag\\
& \leq  \frac{\log(1/(1-2t))}{c_n}
\end{align}
with probability at least $1-\eta_n$. Similarly,
\begin{align}
| S_n(\btheta_0)- \mu_{\btheta_0} | \leq  \frac{\log(1/(1-2t))}{c_n}
\end{align}
with probability at least $1-\eta_n$. Putting these parts together,
\begin{align}
\sup_{\btheta\in R_n} S_n(\btheta) & \leq \sup_{\btheta\in R_n} (S_n(\btheta)- \mu_{\btheta} ) + \sup_{\btheta\in R_n} \mu_{\btheta}		\notag\\
& \leq  \frac{\log(1/(1-2t))}{c_n} + \mu_{\btheta_0} -\epsilon_0\delta_n^2	\notag\\
& \leq S_n(\btheta_0) 	+ \frac{2\log(1/(1-2t))}{c_n} - \epsilon_0 \delta_n^2		\notag\\
& \leq \sup_{\btheta\in \mathcal{B}(\btheta_0; \delta_n)} S_n(\btheta) -  \log(1/(1-2t))/c_n
\end{align}
with probability at least $1-\eta_n$, taking $\delta_n =\sqrt{\frac{3\log(1/(1-2t))}{\epsilon_0 c_n}} $. The required result follows. 
\end{proof}

Further suppose Assumption 5 holds, we can prove Theorem 2.

\begin{proof}[Proof of Theorem 2:]

Let $C(\bx)$ and $\tilde{C}(\bx)$ be the normalizing constants for $\pi^{1/T}(\btheta)$ and $\tilde{\pi}(\btheta)$ respectively. Then

\begin{align} 
\tilde{C}(\bx) & = \int_{\Theta}  \pi^{1/T}(\btheta) \cdot e^{-c_n(\mu(\btheta) - \mu_{\btheta})}U(\btheta) d\btheta	\notag\\
& \leq C(\bx) \sup_{\btheta\in\Theta} e^{-c_n(\mu(\btheta) - \mu_{\btheta})}U(\btheta).
\end{align}
It follows that 
\begin{align}
D_{KL}( \pi_T \Vert  \tilde{\pi} ) & = \int_{\Theta} \pi_0^{1/T}(\btheta) \log\frac{\tilde{C}(x)}{C(x) e^{-c_n(\mu(\btheta) - \mu_{\btheta})}U(\btheta)} d\btheta	\notag\\
& \leq \sup_{\btheta\in\Theta} \left(\log U(\btheta) -c_n(\mu(\btheta)-\mu_{\btheta})\right)	\notag\\
& \qquad - \inf_{\btheta\in\Theta} \left(\log U(\btheta) -c_n(\mu(\btheta)-\mu_{\btheta})\right)	\notag\\
& \leq 2 \sup_{\btheta\in\Theta} |\log U(\btheta) | + 2\sup_{\btheta\in\Theta} c_n |\mu(\btheta)-\mu_{\btheta}|.
\end{align}
The required bound follows by Lemmas~\ref{lem_unif_chernoff} (applied to $m=n$) and ~\ref{lem_u_stat}.

\end{proof}

\section{Applications}
\label{sec_supp_applications}
In this section, we illustrate our assumptions and results in Section~\ref{subsec_theory} can be applied to a number of widely used models in statistics. 

\subsection*{Mixture of exponential family distributions}
We consider the problem of clustering with a $K$-component mixture model of exponential family distributions having a density function, each with parameter $\bphi_k=(\phi_{k,1}, \dots, \phi_{k,p})\in\Phi\subset \R^p$. Let $\balpha=(\alpha_1, \dots, \alpha_K)\in\Lambda$ be the unknown mixture proportions. Then collectively the set of parameters is given by $\btheta=(\alpha_1, \dots, \alpha_K, \phi_1, \dots, \phi_K)\in \Theta=\Lambda \times \Phi^K$. We observe data points $x_1, \dots, x_n$, each drawn independently from the mixture distribution according to some true parameters $\btheta^*$. The goal is to estimate the parameters without observing the class labels of the data points. The likelihood function is given by 
\begin{align}
p_{\btheta}(x) = h(x) \sum_{k=1}^{K} \alpha_k e^{\langle \bphi_k, \mathcal{T}(x)\rangle -A(\bphi_k)}.
\label{eq_density_mixture}
\end{align} 
In this case, we can replace Assumptions~\ref{assum_smooth_x} and~\ref{assum_smooth_theta} with the following conditions. 

\begin{assum} 
Mixture of exponential family distributions.
\begin{enumerate}
\item 
There exists some $\tau>0$ such that for all $\balpha\in\Lambda$, $\min_k \alpha_k > \tau $. 
\item Denote $\bphi=(\bphi_1, \dots, \bphi_K)$, then $\sup_{\bphi \in \Phi^K} \left( \sum_{i=1}^{p} \max_k \phi^2_{k,i}\right)^{1/2} < \infty$. In addition, $\sup_{\phi_{k,i}}|\nabla_{\phi_{k,i}}A(\bphi_k)|$ is bounded for all $i,k$. 
\item $\text{Var}(\log (h(X)))<\infty$.
\item $\btheta^*$ lies in the interior of $\Theta$. 
\end{enumerate}
\label{assump_parameter_bound_mixture}
\end{assum}

Condition 1 ensures the data is from a real $K$-component mixture and there is no model selection issue; conditions 2 and 4 are commonly used regularity conditions; condition 3 is satisfied by many commonly occurring exponential family distributions including multivariate Gaussian, chi-squared distribution, and gamma distribution. First note that $h(\cdot)$ introduces an extra $\log h(x)$ in $\log p_{\btheta}(x)$, which can be handled in Lemma~\ref{lem_unif_chernoff} using standard concentration inequalities such as Bernstein's inequality using condition 3. Since the term is data dependent only, the convergence rate is dominated by the rest of $\log p_{\btheta}(x)$ that depends on $\btheta$. For convenience, we will omit $h(x)$ from now on. 

To check Assumption~\ref{assum_smooth_x}, let
$$
f({\bf t}, \btheta) = \log \left( \sum_{k=1}^{K} \alpha_k e^{\langle \bphi_k, {\bf t}\rangle -A(\bphi_k)} \right).
$$
Taking the derivative with respect to ${\bf t}$, it is easy to check that 
\begin{align*}
| \log p_{\btheta}(x) - \log p_{\btheta}(y) | & \leq \left( \sum_{i=1}^{p} \max_k \phi^2_{k,i}\right)^{1/2} \Vert \mathcal{T}(x)-\mathcal{T}(y) \Vert_2		\\
& \leq \left( \sum_{i=1}^{p} \max_k \phi^2_{k,i}\right)^{1/2} \Vert \mathcal{T}(x)-\mathcal{T}(y) \Vert_1,
\end{align*}
where $L(\bphi):=\left( \sum_{i=1}^{p} \max_k \phi^2_{k,i}\right)^{1/2}$, $\sup_{\bphi\in\Phi^K}L(\bphi) < \infty$ by condition 2 in Assumption~\ref{assump_parameter_bound_mixture}. Furthermore, there exists $\delta_1>0$ such that 
\begin{align*}
\E_{\btheta^*} e^{\delta_1 \Vert \mathcal{T}(X) \Vert_1} \leq \prod_{i=1}^p \left( \E_{\btheta^*} e^{\delta_1 \mathcal{T}_i(X) } + \E_{\btheta^*} e^{-\delta_1 \mathcal{T}_i(X) } \right) <\infty
\end{align*}
by condition 4 in Assumption~\ref{assump_parameter_bound_mixture}. 

To see that Assumption~\ref{assum_smooth_theta} holds, similarly taking the derivative of $f({\bf t}, \btheta)$ with respect to $\btheta$, it is easy to check
\begin{align*}
| \log p_{\btheta}(x) - \log p_{\btheta'}(x) | & \leq \Vert \btheta-\btheta' \Vert_2 \left( \sum_{k=1}^{K}\sum_{i=1}^{p}( |\mathcal{T}_i(x)| + \sup_{\phi_{k,i}}|\nabla_{\phi_{k,i}}A(\bphi_k)|) + K\tau^{-1} \right)	\\
 &:= M(x) \Vert \btheta-\btheta' \Vert_2. 
\end{align*}
By condition 2 and 3, there exists $\delta_2>0$ such that $\E_{\btheta^*} e^{\delta_2 M(X)} < \infty$. 

\subsection*{Linear regression}
In linear regression, we observe $n$ data points with $z_1=(y_1, x_1), \dots, z_n=(y_n, x_n)$, $y_i\in\R$, $x_i\in\R^d$. We have $y_i=\langle \btheta, x_i \rangle+ \epsilon_i$, where $\epsilon_i$ are iid Gaussian noise with unknown variance $\sigma^2$. Here both $\btheta$ and $\sigma$ are the parameters. We consider $x_i$ as feature vectors generated iid from some likelihood $p_0(\cdot)$, which does not depend on the parameters $\btheta$ or $\sigma$. The likelihood function for a data point $(x,y)$ is given by
\begin{align}
p_{\btheta, \sigma}(x,y) & = p_0(x)p_{\btheta, \sigma}(y|x) = \frac{1}{\sqrt{2\pi\sigma^2}} e^{-\frac{(y-\langle \btheta, x \rangle)^2}{2\sigma^2}} p_0(x). 	\notag\\
\end{align}
We assume the following conditions hold.
\begin{assum}
	Linear regression. 
	\begin{enumerate}
		\item 
		$\inf_{\sigma} \sigma >0$ and $\sup_{\sigma} \sigma< \infty$.
		\item
		$\sup_{\btheta} \Vert \btheta \Vert_2 < \infty$. 
		\item 
		The likelihood $p_0(x)$ satisfies Assumption~\ref{assum_smooth_x}, but with the Lipschitz constant independent of $\btheta$ and $\sigma$. The feature vector is bounded in the sense that $\|x\|_2 < \infty$.
	\end{enumerate}
	\label{assump_mgf_feature}
\end{assum}
Under Assumption~\ref{assump_mgf_feature}, it is easy to verify Assumptions~\ref{assum_smooth_x} and~\ref{assum_smooth_theta}. 

\subsection*{Classification with fully connected neural networks}
We are given $n$ data points $z_1=(y_1, x_1), \dots, z_n=(y_n, x_n)$, where $y_i\in\{1,\dots, K\}$ are labels and $x_i\in\R^q$ are features (e.g. pixels in images) generated iid from some likelihood $p_0(\cdot)$, which does not depend on $\btheta$. We consider the popular deep learning classification task with $N$ fully connected layers. In the $\ell$-th layer, the input $x^{(\ell-1)}$ undergoes an affine transformation followed by a nonlinear transformation by an activation function $\sigma(\cdot)$. The output of the $\ell$-th layer is then given by 
$$
x^{(\ell)} = \sigma(W^{(\ell)}x^{(\ell-1)} + b^{(\ell)}),	\qquad \ell=1, \dots, N-1.
$$
where $W^{(\ell)}$ is the weight matrix, $b^{(\ell)}$ is the bias vector in the $\ell$-th layer. Here $x^{(0)}$ corresponds to the input feature vector; the last layer is the softmax function 
$$
 x^{(N)}_k = \frac{\exp(W_{k,\cdot}^{(N)}x^{(N-1)}+b_k^{(N)})}{\sum_{j=1}^{K} \exp(W_{j,\cdot}^{(N)}x^{(N-1)}+b_j^{(N)})}  	\qquad k=1, \dots, K
$$
with $W^{(N)}_{j,\cdot}$ being the $j$-th row of $W^{(N)}$. $x^{(N)}$ can be interpreted as prediction probabilities. Training a neural network involves minimizing some loss function between the labels $\by=(y_1, \dots, y_n)$ and the predictions. We consider the commonly used cross entropy loss, 
$$
H(\by, \btheta)=-\sum_{i=1}^{n}\sum_{k=1}^{K} \bone(y_i=k) \log x_k^{(N)},
$$
where $\btheta$ is the collection of $W^{(\ell)}$, $b^{(\ell)}$, $\ell=1, \dots, N$. In this way we can interpret $-H(\by, \btheta)$ as the sum of log likelihood $p_{\btheta}(y|x)$, with $y$ coming from a multinomial distribution with parameters specified by $\btheta$ and the features $x$. The logistic regression is a special case of this. 

Next we show that Assumptions~\ref{assum_smooth_x} and~\ref{assum_smooth_theta} are satisfied if the following hold.

\begin{assum} 
	\label{assum_nn_classification}
	Activation function and operator norms of weight matrices.
	\begin{enumerate}
	\item The activation function $\sigma(\cdot)$ is bounded and Lipschitz continuous in $\ell_2$ norm.
	
	\item $\sup_{W^{(1)},\dots, W^{(N)}}\prod_{\ell=1}^{N}\lambda^{(\ell)} < \infty$, where $\lambda^{(\ell)} = \| W^{(\ell)}\|_{op}$.
	
	\item $\sup_{b^{(N)}} \|b^{(N)}\|_{\infty} < \infty$.
	
	\item The likelihood $p_0(x)$ satisfies Assumption~\ref{assum_smooth_x}, but with the Lipschitz constant independent of $\btheta$ and $\sigma$. $\|x\| < \infty$.
	\end{enumerate}
\end{assum}
The first condition is satisfied by a wide class of activation functions, including the sigmoid function and other hyperbolic functions. For simplicity, we assume $\|\sigma\|_{\infty}\leq 1$ and the Lipschitz constant is 1. In the second condition, the product of the operator norms is commonly used in the complexity measure for neural networks (e.g. \cite{bartlett2017spectrally}). We require all the weight matrices under consideration to have bounded complexity. 

We first check Assumption~\ref{assum_smooth_x}. For $y=k$, $\tilde{y}=j$, $k\neq j$, suppose in the last layer $W^{(N)}\in \R^{K\times w_N}$ ($w_N$ is the width of the layer),
\begin{align}
 & |\log p_{\btheta}(x,y) - \log p_{\btheta}(\tilde{x},\tilde{y}) | 	\notag\\
\leq &  \left |\log \left( \frac{\exp(W_{k,\cdot}^{(N)}x^{(N-1)}+b_k^{(N)})}{\sum_{i=1}^{K} \exp(W_{i,\cdot}^{(N)}x^{(N-1)}+b_i^{(N)})} \right) - \log \left( \frac{\exp(W_{j,\cdot}^{(N)}\tilde{x}^{(N-1)}+b_j^{(N)})}{\sum_{i=1}^{K} \exp(W_{i,\cdot}^{(N)}\tilde{x}^{(N-1)}+b_i^{(N)})} \right)  \right|	\notag\\
& \qquad + |\log p_0(x)-\log p_0(\tilde{x})|	\notag\\
\leq &  |W_{k,\cdot}^{(N)}x^{(N-1)}+b_k^{(N)} - W_{j,\cdot}^{(N)}x^{(N-1)}-b_j^{(N)} |		\notag\\
& \quad + \left |\log \left( \frac{\exp(W_{j,\cdot}^{(N)}x^{(N-1)}+b_j^{(N)})}{\sum_{i=1}^{K} \exp(W_{i,\cdot}^{(N)}x^{(N-1)}+b_i^{(N)})} \right) - \log \left( \frac{\exp(W_{j,\cdot}^{(N)}\tilde{x}^{(N-1)}+b_j^{(N)})}{\sum_{i=1}^{K} \exp(W_{i,\cdot}^{(N)}\tilde{x}^{(N-1)}+b_i^{(N)})} \right)  \right|	\notag\\
& \quad + |\log p_0(x)-\log p_0(\tilde{x})|.
\label{eq_nn_smooth_x}
\end{align}
Next note that $\log ( \frac{\exp(t_k)}{\sum_{j=1}^K \exp(t_j)})$ is Lipschitz in $\ell_2$ norm, using condition 4,~\eqref{eq_nn_smooth_x} is bounded by

\begin{align*}
& \|W_{k,\cdot}^{(N)} -W_{j,\cdot}^{(N)}\|_2 \|x^{(N-1)}\|_2 +  |b_k^{(N)}-b_j^{(N)}|	+ C_1\|\mathcal{T}(x)-\mathcal{T}(\tilde{x})\|_1 + C_2\|W^{(N)}(x^{(N-1)}-\tilde{x}^{(N-1)})\|_2 	\\
\leq & \sqrt{w_N}\|W_{k,\cdot}^{(N)} -W_{j,\cdot}^{(N)}\|_2 + 2\max_k |b_k^{(N)}|	+ 2C_2\lambda^{(N)} \sqrt{w_N}+ C_1\|\mathcal{T}(x)-\mathcal{T}(\tilde{x})\|_1		\\
\leq & 2\sqrt{w_N}\max_k \|W_{k,\cdot}^{(N)}\|_2 + 2\max_k |b_k^{(N)}|	+ 2C_2\lambda^{(N)} \sqrt{w_N}+ C_1\|\mathcal{T}(x)-\mathcal{T}(\tilde{x})\|_1		\\
\leq & 2(w_N\lambda^{(N)}+ C_2\sqrt{w_N} \lambda^{(N)} + \|b^{(N)}\|_{\infty})|y-\tilde{y}| + C_1\|\mathcal{T}(x)-\mathcal{T}(\tilde{x})\|_1, 	
\end{align*}
since $|y-\tilde{y}|\geq 1$. For $y=\tilde{y}=k$,~\eqref{eq_nn_smooth_x} is bounded by 
\begin{align*}
	& C_1\|\mathcal{T}(x)-\mathcal{T}(\tilde{x})\|_1 + C_2\|W^{(N)}(x^{(N-1)}-\tilde{x}^{(N-1)})\|_2 	\\
\leq & C_1\|\mathcal{T}(x)-\mathcal{T}(\tilde{x})\|_1 + C_2\lambda^{(N)} \|W^{(N-1)}(x^{(N-2)}-\tilde{x}^{(N-2)}) \|_2	\\
\leq & C_1\|\mathcal{T}(x)-\mathcal{T}(\tilde{x})\|_1 + C_2\lambda^{(N)}\cdots\lambda^{(1)} \|x-\tilde{x}\|_2,
\end{align*}
by conditions 1-3 in Assumption~\ref{assum_nn_classification}. In either case, Assumption~\ref{assum_smooth_x} is satisfied.

To check Assumption~\ref{assum_smooth_theta}, suppose $y=k$, then
\begin{align*}
	& |\log p_{\btheta}(x,y) - \log p_{\tilde{\btheta}}(x,y) | 	\\
	= & \left |\log \left( \frac{\exp(W_{k,\cdot}^{(N)}x^{(N-1)}+b_k^{(N)})}{\sum_{j=1}^{K} \exp(W_{j,\cdot}^{(N)}x^{(N-1)}+b_j^{(N)})} \right) - \log \left( \frac{\exp(\widetilde{W}_{k,\cdot}^{(N)}\tilde{x}^{(N-1)}+\tilde{b}_k^{(N)})}{\sum_{j=1}^{K} \exp(\widetilde{W}_{j,\cdot}^{(N)}\tilde{x}^{(N-1)}+\tilde{b}_j^{(N)})} \right)  \right|.
\end{align*}
Using the fact that $\log ( \frac{\exp(t_k)}{\sum_{j=1}^K \exp(t_j)})$ is Lipschitz in $\ell_2$ norm, the above is bounded by
\begin{align*}
& C \| W^{(N)}x^{(N-1)} - \widetilde{W}^{(N)}\tilde{x}^{(N-1)} + b^{(N)}-\tilde{b}^{(N)} \|_2	\\
\leq & C( \| W^{(N)} - \widetilde{W}^{(N)} \|_{op} + \|W^{(N)}\|_{op}\|x^{(N-1)} - \tilde{x}^{(N-1)} \|_2 +  \|b^{(N)}-\tilde{b}^{(N)} \|_2)		\\
= & C( \lambda^{(N)}\|x^{(N-1)} - \tilde{x}^{(N-1)} \|_2 + \| W^{(N)} - \widetilde{W}^{(N)} \|_{op} + \|b^{(N)}-\tilde{b}^{(N)} \|_2).
\end{align*}
Continuing the same way with $\|x^{(N-1)} - \tilde{x}^{(N-1)} \|_2$, we can show 
\begin{align*}
&|\log p_{\btheta}(y) - \log p_{\tilde{\btheta}}(y) |  \\
\leq & C\left( \| W^{(N)} - \widetilde{W}^{(N)} \|_{op} + \lambda^{(N)} \| W^{(N-1)} - \widetilde{W}^{(N-1)} \|_{op} + \dots + (\prod_{\ell=2}^{N} \lambda^{(\ell)}) \| W^{(1)} - \widetilde{W}^{(1)} \|_{op} \right.		\\
& \qquad \left. + \|b^{(N)}-\tilde{b}^{(N)} \|_2 + \lambda^{(N)} \|b^{(N-1)}-\tilde{b}^{(N-1)} \|_2 + \dots +  (\prod_{\ell=2}^{N} \lambda^{(\ell)}) \| b^{(1)} - \tilde{b}^{(1)} \|_2 \right)		\\
& \leq C' \|\btheta - \tilde{\btheta}\|_2
\end{align*}
using condition 2 in Assumption~\ref{assum_nn_classification}.

\section{RSGLD with augmented variables}
\label{sec_supp_rsgld}
We first show using augmented variables, MHBT with RSGLD as proposal leads to a tempered MCMC. As before, let $\tau\in\{0,1\}^n$ be an augmented variable with $I(\tau)=\{i: \tau_i=1\}$ and $|I(\tau)|=m$, then we can write $\hat{\mu}_{I(\tau)} = \frac{1}{m}\sum_{i=1}^n \ell_i(\btheta)\tau_i$, $\hat{g}_{I(\tau)} = \frac{1}{m}\sum_{i=1}^n \nabla_{\btheta}\ell_i(\btheta)\tau_i$. For $(\btheta, \tau)$,  consider the proposal 
\begin{align*}
& q((\btheta, \tau)\to (\btheta', \tau'))  = q_{I(\tau)}(\btheta \to \btheta') \nu_{m,n}(\tau') \\
= &  \left\{ \frac{1}{2} \phi \left( \btheta'-\btheta - \epsilon \hat{g}_{I(\tau)}(\btheta);  \frac{2\epsilon}{n^2} I_d \right) +  \frac{1}{2} \phi \left( \btheta'-\btheta + \epsilon \hat{g}_{I(\tau)}(\btheta); \frac{2\epsilon \beta^2}{n^2} I_d \right) \right\} \nu_{m,n}(\tau'),	\\
\end{align*}
and the target distribution
\begin{align*}
\tilde{\pi}(\btheta, \tau) \propto e^{c_n\hat{\mu}_{I(\tau)}(\btheta)} \nu_{m,n}(\tau).
\end{align*}
Then the acceptance probability is given by 
\begin{align}
r((\btheta, \tau) \to (\btheta', \tau')) & = \min \left\{ 1, \frac{\tilde{\pi}(\btheta', \tau') q((\btheta', \tau') \to (\btheta, \tau))}{\tilde{\pi}(\btheta, \tau) q((\btheta, \tau) \to (\btheta', \tau'))} \right\}	\notag\\
 & = \min \left\{ 1, \frac{ q_{I(\tau')}(\btheta' \to \btheta) \nu_{m,n} e^{c_n \hat{\mu}_{I(\tau')}(\btheta')} \nu_{m,n} }{ q_{I(\tau)}(\btheta \to \btheta') \nu_{m,n} e^{c_n\hat{\mu}_{I(\tau)}(\btheta)} \nu_{m,n}} \right\}	\notag\\
 &=  \min \left\{ 1, \frac{ q_{I(\tau')}(\btheta' \to \btheta) e^{c_n \hat{\mu}_{I(\tau')}(\btheta')} }{ q_{I(\tau)}(\btheta \to \btheta')  e^{c_n\hat{\mu}_{I(\tau)}(\btheta)} } \right\}
\end{align}
which is exactly~\eqref{eq_accept_prob_mas}. 
\section{Improved acceptance probability with RSGLD}
\label{sec_supp_acceptprob}

We now give the proof of Proposition~\ref{prop_accept_prob}, which calculates the proposal ratio of RSGLD. 

\begin{proof}[Proof of Proposition~\ref{prop_accept_prob}]

\noindent \textbf{Case 1) and a forward move in~\eqref{eq_proposal}.}
  
 In this case,
\begin{align}
& \frac{q_{J}(\btheta' \to \btheta)}{q_{I}(\btheta \to \btheta')}	\notag\\
= & \frac{ e^{-\frac{1}{2\beta^2} \left\Vert n\sqrt{\epsilon/2}\left(\hat{g}_J(\btheta')-\hat{g}_I(\btheta) \right) - \bZ\right\Vert_2^2} + e^{-\frac{1}{2} \left\Vert n\sqrt{\epsilon/2}\left(\hat{g}_J(\btheta')+\hat{g}_I(\btheta) \right) + \bZ\right\Vert_2^2} }{e^{-\frac{1}{2} \Vert \bZ \Vert_2^2} + e^{-\frac{1}{2\beta^2} \Vert \sqrt{2\epsilon} n \hat{g}_I(\btheta) + \bZ \Vert_2^2}}.
\label{eq_proposal_ratio}
\end{align}
In the exponent of the denominator, 
\begin{align}
&\frac{1}{\beta^2}\Vert \sqrt{2\epsilon} n \hat{g}_I(\btheta) + \bZ \Vert_2^2 - \Vert \bZ \Vert_2^2		\notag\\
= & \frac{2}{\beta^2} \epsilon n^2 \Vert \hat{g}_I(\btheta) \Vert_2^2 + \frac{2}{\beta^2} \sqrt{2\epsilon}n \hat{g}_I(\btheta)^T \bZ + (1/\beta^2 -1) \Vert \bZ \Vert_2^2 	\notag\\
\geq & \epsilon n^2\|\hat{g}_I(\btheta) \|_2^2(\frac{2}{\beta^2}- \eta_0\frac{2}{\beta^2} - \frac{\eta_0^2}{2}(1-\frac{1}{\beta^2}))	\\
\geq &  C \epsilon n^2 \Vert \hat{g}_I(\btheta) \Vert_2^2  = \Omega_P(d),
\end{align}
for some positive constant $C$ since $\epsilon n^2 \Vert \hat{g}_I(\btheta) \Vert_2^2 \geq \frac{2}{\eta_0^2}\| \bZ \|_2^2$ with high probability. Here we have used
\begin{align}
\left\vert \sqrt{2\epsilon}n \hat{g}_I(\btheta)^T \bZ \right\vert  & \leq \sqrt{2\epsilon}n \Vert \hat{g}_I(\btheta)\Vert_2 \Vert \bZ \Vert_2 \leq \eta_0\epsilon n^2 \Vert \hat{g}_I(\btheta)\Vert_2^2,		\notag\\
\Vert \bZ \Vert^2_2 & \leq \eta^2_0/2 \cdot \epsilon n^2  \Vert \hat{g}_I(\btheta)\Vert^2_2	\qquad \text{ w.h.p.}
\label{eq_crossterm_cs}
\end{align}
for $\eta_0$ small using the condition in Case 1). Thus the denominator in~\eqref{eq_proposal_ratio} is $e^{-1/2\Vert \bZ \Vert_2^2} (1+O(e^{-\Omega_P(d)})) = e^{-1/2\Vert \bZ \Vert_2^2} (1+o_P(1))$.

Similarly in the numerator, the first term dominates. Since $\beta> 1$,
\begin{align}
& \left\Vert n\sqrt{\epsilon/2}\left(\hat{g}_J(\btheta')+\hat{g}_I(\btheta) \right) + \bZ\right\Vert_2^2- \frac{1}{\beta^2}\left\Vert n\sqrt{\epsilon/2}\left(\hat{g}_J(\btheta')-\hat{g}_I(\btheta) \right) - \bZ\right\Vert_2^2 		\notag\\
\geq  &  \left\Vert n\sqrt{\epsilon/2}\left(\hat{g}_J(\btheta')+\hat{g}_I(\btheta) \right) + \bZ\right\Vert_2^2- \left\Vert n\sqrt{\epsilon/2}\left(\hat{g}_J(\btheta')-\hat{g}_I(\btheta) \right) - \bZ\right\Vert_2^2 		\notag\\
= & \frac{\epsilon n^2}{2} \left( \Vert \hat{g}_J(\btheta')+\hat{g}_I(\btheta) \Vert_2^2 -  \Vert \hat{g}_J(\btheta')-\hat{g}_I(\btheta) \Vert_2^2 \right) 	+ 2n\sqrt{2\epsilon} \hat{g}_J(\btheta')^T\bZ	\notag\\
= & 2 n^2 \hat{g}_J(\btheta')^T\left(\epsilon\hat{g}_I(\btheta) +  \frac{\sqrt{2\epsilon}}{n}\bZ\right),
\label{eq_numerator_proposal_ratio}
\end{align}
where $\hat{g}_J(\btheta') = \hat{g}_J(\btheta) + H_J(\btheta_0)(\epsilon  \hat{g}_I(\btheta) + \frac{\sqrt{2\epsilon}}{n}\bZ)$ for some $\btheta_0$ between $\btheta$ and $\btheta'$. 
We can lower bound this term by noting that
\begin{align*}
	\hat{g}_J(\btheta')^T \hat{g}_I(\btheta) & = (\hat{g}_J(\btheta') - \hat{g}_I(\btheta))^T \hat{g}_I(\btheta) + \|\hat{g}_I(\btheta)\|_2^2,
\end{align*}
where by Assumption~\ref{assum_spec_norm}, \ref{assum_similar_grad}, and Eq~\eqref{eq_crossterm_cs},
\begin{align}
& \Vert \hat{g}_J(\btheta')-\hat{g}_I(\btheta)\Vert_2^2 		\notag\\
\leq & 2\Vert \hat{g}_J(\btheta)-\hat{g}_I(\btheta) \Vert_2^2 + 2\Vert H_J(\btheta_0)(\epsilon \hat{g}_I(\btheta) +\sqrt{2\epsilon}/n \bZ) \Vert_2^2	\notag\\
\leq & 2\Vert \hat{g}_J(\btheta)-\hat{g}_I(\btheta) \Vert_2^2 + 2\epsilon^2 \lambda^2 \Vert \hat{g}_I(\btheta)\Vert_2^2 + \frac{4\epsilon \lambda^2}{n^2} \Vert \bZ \Vert_2^2	\notag\\
\leq & C \epsilon^2 \Vert \hat{g}_I(\btheta)\Vert_2^2,
\label{eq_case1_part1}
\end{align}
w.h.p., thus
\begin{align}
\hat{g}_J(\btheta')^T \hat{g}_I(\btheta) \geq (1-C\epsilon^2) \Vert \hat{g}_I(\btheta)\Vert_2^2.
\label{eq_IJcross}
\end{align}
By~\eqref{eq_crossterm_cs} and~\eqref{eq_case1_part1},
\begin{align}
|\hat{g}_J(\btheta')^T\bZ| & \leq \|\hat{g}_J(\btheta') - \hat{g}_I(\btheta)\|_2\|\bZ\|_2 + \|\hat{g}_I(\btheta)\|_2\|\bZ\|_2	\notag\\
& \leq (1+C\epsilon)\|\hat{g}_I(\btheta)\|_2\|\bZ\|_2	\notag\\
& \leq n\sqrt{\frac{\epsilon}{2}} \eta_0(1+C\epsilon) \|\hat{g}_I(\btheta)\|_2^2
\label{eq_case1_part1_}
\end{align}
w.h.p. Eq~\eqref{eq_IJcross} and~\eqref{eq_case1_part1_} imply~\eqref{eq_numerator_proposal_ratio} is lower bounded by
\begin{align*}
2n^2\epsilon(1-C\epsilon^2-\eta_0(1+C\epsilon)) \|\hat{g}_I(\btheta)\|_2^2 \geq 2C_1 n^2\epsilon \|\hat{g}_I(\btheta)\|_2^2 = \Omega_P(d)
\end{align*}
for $\eta_0$ and $\epsilon$ small, and the last part follows from~\eqref{eq_crossterm_cs}. Hence the numerator in~\eqref{eq_proposal_ratio} is $(1+o_P(1))e^{-\frac{1}{2\beta^2} \left\Vert n\sqrt{\epsilon/2}\left(\hat{g}_J(\btheta')-\hat{g}_I(\btheta) \right) - \bZ\right\Vert_2^2} $.


The above approximations show~\eqref{eq_proposal_ratio} can be written as  
\begin{align}
& \frac{e^{-\frac{1}{2\beta^2} \left\Vert n\sqrt{\epsilon/2}\left(\hat{g}_J(\btheta')-\hat{g}_I(\btheta) \right) - \bZ\right\Vert_2^2}}{e^{-\frac{1}{2} \Vert \bZ \Vert_2^2}} (1+o_P(1))	\notag\\
= & (1+o_P(1))\exp \left\{ -\frac{n^2\epsilon}{4\beta^2} \Vert \hat{g}_J(\btheta')-\hat{g}_I(\btheta)\Vert_2^2 + \frac{n}{\beta^2}\sqrt{\frac{\epsilon}{2}} (\hat{g}_J(\btheta')-\hat{g}_I(\btheta))^T\bZ + \frac{1}{2}(1-1/\beta^2) \Vert \bZ \Vert_2^2 \right\}	\notag\\
\geq & (1+o_P(1))\exp\left\{ -C\frac{n^2\epsilon^2(\epsilon \vee \eta_0)}{\beta^2} \Vert \hat{g}_I(\btheta) \Vert_2^2 + \frac{1}{2}(1-1/\beta^2)\Vert \bZ \Vert_2^2 \right\} > 1	\qquad \text{w.h.p.}
\label{eq_proposal_ratio_approx}
\end{align}
by~\eqref{eq_crossterm_cs} and~\eqref{eq_case1_part1} again.

\begin{align}
& \exp\left\{ -C\frac{n^2\epsilon^2(\epsilon \vee \eta_0)}{\beta^2} \Vert \hat{g}_I(\btheta) \Vert_2^2 + \frac{1}{2}(1-1/\beta^2)\Vert \bZ \Vert_2^2 \right\} > 1	\qquad \text{w.h.p.}
\end{align}
when $\beta >1$, since $\frac{n^2\epsilon^2(\epsilon \vee \eta_0)}{\beta^2-1} \Vert \hat{g}_I(\btheta) \Vert_2^2 = o(d)$ in this case. 

\noindent \textbf{Case 1) and a backward move in~\eqref{eq_proposal}.}

In this case, 
\begin{align}
& \frac{q_{J}(\btheta' \to \btheta)}{q_{I}(\btheta \to \btheta')}	\notag\\
= & \frac{ e^{-\frac{1}{2} \left\Vert -n\sqrt{\epsilon/2}\left(\hat{g}_J(\btheta')-\hat{g}_I(\btheta) \right) - \beta \bZ\right\Vert_2^2} + e^{-\frac{1}{2\beta^2} \left\Vert n\sqrt{\epsilon/2}\left(\hat{g}_J(\btheta')+\hat{g}_I(\btheta) \right) - 
\beta \bZ\right\Vert_2^2} }{e^{-\frac{1}{2} \Vert \bZ \Vert_2^2} + e^{-\frac{1}{2} \Vert -\sqrt{2\epsilon} n \hat{g}_I(\btheta) + \beta \bZ \Vert_2^2}}	\notag\\
= & \frac{e^{-\frac{1}{2} \left\Vert -n\sqrt{\epsilon/2}\left(\hat{g}_J(\btheta')-\hat{g}_I(\btheta) \right) - \beta \bZ\right\Vert_2^2} }{e^{-\frac{1}{2} \Vert \bZ \Vert_2^2}} (1+o_P(1))	
\label{eq_proposal_ratio_case2}
\end{align}
by similar arguments as above. From~\eqref{eq_proposal_ratio_case2}, we have
\begin{align}
& \exp\left\{ -\frac{1}{2}(\beta^2-1)\Vert \bZ\Vert_2^2 -\frac{n^2\epsilon}{4} \Vert \hat{g}_J(\btheta')-\hat{g}_I(\btheta) \Vert_2^2 - n\sqrt{\frac{\epsilon}{2}}\beta(\hat{g}_J(\btheta')-\hat{g}_I(\btheta))^T\bZ  \right\}	\notag\\
\leq & \exp\left\{ -\frac{1}{2}(\beta^2-1)\Vert \bZ\Vert_2^2 - n\sqrt{\frac{\epsilon}{2}}\beta(\hat{g}_J(\btheta')-\hat{g}_I(\btheta))^T\bZ  \right\} 	\notag\\
\leq & \exp\left\{ -\frac{1}{2}(\beta^2-1)\Vert \bZ\Vert_2^2 + Cn^2 \epsilon^2\eta_0 \Vert \hat{g}_I(\btheta) \Vert_2^2 \right\}  = o_P(1), 
\end{align}
where we have used~\eqref{eq_crossterm_cs} and~\eqref{eq_case1_part1}, and the condition $\frac{n^2\epsilon^2\eta_0}{\beta^2-1} \Vert \hat{g}_I(\btheta) \Vert_2^2 = o(d)$  in Case 1). 

\noindent \textbf{Case 2) and a forward move in~\eqref{eq_proposal}.}

In this case we have
\begin{align}
& \frac{q_{J}(\btheta' \to \btheta)}{q_{I}(\btheta \to \btheta')}	\notag\\
= & \frac{ e^{-\frac{1}{2\beta^2} \left\Vert n\sqrt{\epsilon/2} \hat{g}_J(\btheta')  - \bZ\right\Vert_2^2} + e^{-\frac{1}{2} \left\Vert n\sqrt{\epsilon/2} \hat{g}_J(\btheta') + \bZ\right\Vert_2^2} }{e^{-\frac{1}{2} \Vert \bZ \Vert_2^2} + e^{-\frac{1}{2\beta^2} \Vert  \bZ \Vert_2^2}}.
\label{eq_proposal_ratio_case3}
\end{align}
Noting that $\hat{g}_J(\btheta')=\hat{g}_J(\btheta) + \frac{\sqrt{2\epsilon}}{n} H_J(\btheta_0) \bZ=\frac{\sqrt{2\epsilon}}{n} H_J(\btheta_0) \bZ$ by Assumption~\ref{assum_similar_grad}, 
\begin{align}
& \left\Vert n\sqrt{\epsilon/2} \hat{g}_J(\btheta')  - \bZ\right\Vert_2^2	\notag\\
= & \frac{n^2\epsilon}{2} \Vert  \hat{g}_J(\btheta') \Vert_2^2 + \Vert \bZ \Vert_2^2 -n\sqrt{2\epsilon}  \hat{g}_J(\btheta')^T\bZ,
\label{eq_part1_case3}
\end{align}
where $\Vert \hat{g}_J(\btheta') \Vert_2^2 \leq \frac{2\epsilon}{n^2} \lambda^2 \Vert \bZ\Vert^2_2$, $|\hat{g}_J(\btheta')^T\bZ| \leq \frac{\sqrt{2\epsilon}}{n}\lambda \Vert \bZ\Vert_2^2$. It follows then~\eqref{eq_part1_case3} is of the same order as $(1+O(\epsilon))\Vert \bZ \Vert_2^2$, where $\epsilon\Vert \bZ \Vert_2^2=o_P(1)$ using the condition on $\epsilon$ in Case 2). The same argument holds for $\left\Vert n\sqrt{\epsilon/2} \hat{g}_J(\btheta')  + \bZ\right\Vert_2^2$. Hence~\eqref{eq_proposal_ratio_case3} is 
\begin{align*}
\frac{ e^{-\frac{1}{2\beta^2} \left\Vert \bZ\right\Vert_2^2} + e^{-\frac{1}{2} \left\Vert \bZ\right\Vert_2^2} }{e^{-\frac{1}{2} \Vert \bZ \Vert_2^2} + e^{-\frac{1}{2\beta^2} \Vert  \bZ \Vert_2^2}} (1 + o_P(1)).
\end{align*}
The proposal behaves like a random walk. 

\noindent \textbf{Case 2) and a backward move in~\eqref{eq_proposal}.}

\begin{align}
& \frac{q_{J}(\btheta' \to \btheta)}{q_{I}(\btheta \to \btheta')}	\notag\\
= & \frac{ e^{-\frac{1}{2} \left\Vert -n\sqrt{\epsilon/2} \hat{g}_J(\btheta') - \beta \bZ\right\Vert_2^2} + e^{-\frac{1}{2\beta^2} \left\Vert n\sqrt{\epsilon/2} \hat{g}_J(\btheta') - 
\beta \bZ\right\Vert_2^2} }{e^{-\frac{1}{2} \Vert \bZ \Vert_2^2} + e^{-\frac{1}{2} \Vert \beta \bZ \Vert_2^2}}.
\end{align}
The same arguments as above can be used to show this ratio is approximately 1. 
\end{proof}

\section{Convolutional neural network for CIFAR-10}
\label{sec_supp_nn}


\begin{table}[ht]
\centering 
\begin{tabular}{c c c c }
Type of layer & Number of filters & Filter size / stride & Output size 	\\
\hline
Convolution & 32 & $5\times 5$ / 1 & $32\times32\times32$ 	\\
ReLU & & 	\\
Max pooling &  & $3\times 3$ / 2 & 	$16\times16\times32$	\\
LRN & 	\\
Convolution & 32 & $5\times 5$ / 1 & $16\times16\times32$	\\
ReLU & 	\\
Max pooling &  & $3\times 3$ / 2 & $8\times8\times32$		\\
LRN & 	\\
Convolution & 64 & $5\times 5$ / 1 &  $8\times8\times64$	\\
ReLU & 	\\
Max pooling &  & $3\times 3$ / 2 & $4\times4\times64$		\\
LRN & 	\\
Fully-connected & & & 10		\\
\hline
\end{tabular}
\caption{Architecture of the 3-layer CNN used on the CIFAR-10 dataset. All Local Response Normalization (LRN) layers used depth radius=3,  bias=1,  alpha=$5\times10^{-5}$,  beta=0.75.}
\label{tab_cnn_structure_cifar10}
\end{table}

\end{appendices}

\bibliographystyle{plain}
\bibliography{ref}
\end{document}